\documentclass[accepted]{uai2022} 
\usepackage[american]{babel}

\usepackage{natbib} 
\usepackage{mathtools} 
\usepackage{booktabs} 
\usepackage{tikz} 

\usepackage{xcolor}
\usepackage{amsmath, amssymb, amsthm}
\usepackage{mathtools}
\newtheorem{theorem}{Theorem}

\newtheorem{corollary}{Corollary}

\newtheorem{lemma}{Lemma}
\newtheorem{definition}{Definition}

\DeclareMathOperator*{\argmax}{arg\,max}
\DeclareMathOperator*{\argmin}{arg\,min}

\DeclareMathOperator*{\defn}{ \ \overset{\mathrm{def}}{=} \ }

\newcommand\blfootnote[1]{%
  \begin{NoHyper}%
  \renewcommand\thefootnote{}\footnote{#1}%
  \addtocounter{footnote}{-1}%
  \end{NoHyper}%
}

\title{PAC-Bayesian Domain Adaptation Bounds for Multiclass Learners}

\author[1]{\href{mailto:<anthonysicilia@pitt.edu>?Subject=Your UAI 2022 paper}{Anthony Sicilia}{}}
\author[2]{\href{mailto:<kaa139@pitt.edu>?Subject=Your UAI 2022 paper}{Katherine Atwell}{}}
\author[1,2]{\href{mailto:<malihe@pitt.edu>?Subject=Your UAI 2022 paper}{Malihe Alikhani}{}}
\author[3*]{\href{mailto:<seongjae@yonsei.ac.kr>?Subject=Your UAI 2022 paper}{Seong Jae Hwang}{}}
\affil[1]{%
    Intelligent Systems Program\\
    University of Pittsburgh\\
    Pittsburgh, Pennsylvania, USA
}
\affil[2]{%
    Department of Computer Science\\
    University of Pittsburgh\\
    Pittsburgh, Pennsylvania, USA
}
\affil[3]{%
    Department of Artificial Intelligence\\ Yonsei University\\
    Seoul, South Korea
  }
  
\begin{document}
\maketitle

\begin{abstract}
Multiclass neural networks are a common tool in modern unsupervised domain adaptation, yet an appropriate theoretical description for their non-uniform sample complexity is lacking in the adaptation literature. To fill this gap, we propose the first PAC-Bayesian adaptation bounds for multiclass learners. We facilitate practical use of our bounds by also proposing the first approximation techniques for the multiclass distribution divergences we consider. For divergences dependent on a Gibbs predictor, we propose additional PAC-Bayesian adaptation bounds which remove the need for inefficient Monte-Carlo estimation. Empirically, we test the efficacy of our proposed approximation techniques as well as some novel design-concepts which we include in our bounds. Finally, we apply our bounds to analyze a common adaptation algorithm that uses neural networks.
\end{abstract}

\section{Introduction}
\label{sec:intro}
\blfootnote{$^*$Corresponding Author}Multiclass neural networks are frequently used in implementation of many unsupervised domain adaptation algorithms. For example, neural networks are often employed for invariant feature learning algorithms \citep{ganin2015unsupervised, long2017deep, long2018conditional, zhang2019bridging}, importance weighting algorithms \citep{lipton2018detecting}, or combinations of both techniques \citep{tachet2020domain}. While most of these adaptation algorithms are motivated by theoretical bounds, recent literature has paid close attention to the assumptions and failure-cases of some techniques \citep{zhao2019learning, wu2019domain, johansson2019support}. Namely, some learning algorithms \textit{ignore} key terms in the adaptation bounds on which they are based, and as a result, may output solutions (i.e., learned models) that violate assumptions and are \textit{guaranteed} to fail at the adaptation task \citep{zhao2019learning, wu2019domain}. Still, the story here is not totally complete. In particular, there has not been much discussion of the non-uniform sample complexity of these modern adaptation algorithms. Sample complexity, in fact, contributes an additional ``ignored'' term in the theoretical bounds on which modern adaptation algorithms are based. 

In this paper, we propose the first multiclass adaptation bounds which allow us to study this non-uniform sample complexity. Studying sample complexity is important to our understanding of adaptation algorithms because it describes how ``data-hungry'' an algorithm is. When this sample complexity is non-uniform across an algorithm's solution space, it allows us to study properties of a solution as a function of its ``data-hunger.'' This is especially important for adaptation algorithms, which as mentioned, can inadvertently output poor solutions. Identifying a dynamic relationship between the properties of solutions and their non-uniform sample-complexity can provide insight on how to prevent these failure-cases in practice (e.g., by collecting sufficient data for an algorithm).
Non-uniform sample complexity (rather than uniform complexity) can also help us to better quantify implicit regularization inherent to our algorithm \citep{dziugaite2017computing, nagarajan2019uniform}. Accurately describing implicit regularization is especially important when using neural networks \citep{neyshabur2014search, neyshabur2017implicit, keskar2017large, zhang2017understanding}, since similar learning algorithms can lead to solutions with distinct generalization performance and implicit regularization is believed to be the cause of this phenomena.

Despite the importance of studying non-uniform sample complexity in modern adaptation contexts, we are not aware of any multiclass adaptation bounds with this capability. To fill this gap, we contribute the first PAC-Bayesian adaptation bound for multiclass learners (Thm.~\ref{thm:pb-bound}). While PAC-Bayesian bounds actually control error for \textit{stochastic} models, we choose this framework for its demonstrated empirical accuracy in describing neural network sample complexity \citep{dziugaite2017computing, zhou2018non, jiang2019fantastic, dziugaite2020search, dziugaite2021role, perez2021tighter}. Compared to existing bounds, we design our proposals to be more sensitive to the solution output by our learning algorithm as well as the data sample available for estimating key quantities. The former is vital in studying non-uniform complexity of adaptation algorithms (as discussed), while the latter is important for facilitating empirical study. To make our bound useful in practice, we also propose the first approximation techniques for the divergence terms in our bound. In one case, this involves proposal of a novel surrogate for optimizing 01-loss (Thm.~\ref{thm:surrogate_loss}). In another, we show a standard technique for computing divergence fails to generalize to the mutliclass setting without additional constraints (Thm.~\ref{thm:mdp_div_red2erm}). Working in the PAC-Bayesian framework, some divergences we study are also expressed as expectations with no known analytic solution. For these, we propose additional bounds (Thm.~\ref{thm:pb-bound-efficient}, Cor.~\ref{cor:pb-bound-efficient}) which allow us to avoid inefficient Monte-Carlo estimation by introduction of a new flatness assumption related to the well-known flat-minima hypothesis \citep{hochreiter1997flat}. To conclude, we conduct extensive empirical study of more than 12K models learned across 5 diverse adaptation datasets. 
\section{Background}
\label{sec:background}
\subsection{Notation and Assumptions}
Consider the space $\mathcal{X} \times \mathcal{Y}$ for some finite $\mathcal{Y}$ with $|\mathcal{Y}| > 2$ unless otherwise noted. Colloquially, we call $\mathcal{X}$ the feature space and $\mathcal{Y}$ the label space. For a distribution $\mathbb{D}$ over $\mathcal{X} \times \mathcal{Y}$, we are interested in the risk functional $\mathbf{R}_\mathbb{D}: \mathcal{Y}^\mathcal{X} \to [0,1]$ 
\begin{equation}\small
\label{eqn:risk}
\mathbf{R}_\mathbb{D}(h) \defn \mathbf{Pr}(h(X) \neq Y); \qquad (X,Y) \sim \mathbb{D}
\end{equation}
applied to some hypothesis (i.e., model) $h \in \mathcal{H} \subseteq \mathcal{Y}^\mathcal{X}$. The risk functional $\mathbf{R}_\mathbb{D}$ precisely gives the error rate of the hypothesis $h$ when tasked with modelling the relationship between $\mathcal{X}$ and $\mathcal{Y}$ described by $\mathbb{D}$. In PAC-Bayes, we also consider the risk of stochastic (Gibbs) predictors. For a distribution $\mathbb{Q}$ over $\mathcal{H} \subseteq \mathcal{Y}^\mathcal{X}$, the Gibbs risk is the expectation
\begin{equation}\small
\label{eqn:gibbs_risk}
    \mathbf{R}_\mathbb{D}(\mathbb{Q}) \defn \mathbf{E}[\mathbf{R}_\mathbb{D}(H)]; \qquad H \sim \mathbb{Q}.
\end{equation}
For neural networks, a common stochastic formulation is to sample weights from the distribution $\mathbb{Q}$ before inference -- e.g., the Bayesian neural networks of \citet{blundell2015weight}.

Throughout this paper, we assume a source distribution $\mathbb{S}$ over $\mathcal{X} \times \mathcal{Y}$ and a target distribution $\mathbb{T}$ over $\mathcal{X} \times \mathcal{Y}$. We assume observation of an i.i.d. random sample $S \sim \mathbb{S}^n$ and an i.i.d. random sample $T_X \sim \mathbb{T}_X^m$ where the subscript $X$ denotes the $\mathcal{X}$-marginal of a distribution. In this context, an algorithm for the unsupervised adaptation problem we study is a function $(S, T_X) \mapsto h \in \mathcal{H}$. We are interested in bounds on $\mathbf{R}_\mathbb{T}(h)$ for such algorithms. 

Interchangeably, we think of the sample $S$ as both a random variable with distribution $\mathbb{S}^n$ and a (random) distribution itself, since any observation of a sample $S$ uniquely defines its own empirical distribution over $\mathcal{X} \times \mathcal{Y}$ by the pmf
\begin{equation}\small
\label{eqn:sample_pmf}
    (x,y) \mapsto  n^{-1}\sum\nolimits_{i=1}^n \mathbf{1}_{\{(x_i, y_i)\}}\{(x,y)\}
\end{equation}
where $\mathbf{1}$ is the indicator function. So, $\mathbf{R}_S$ is well-defined by this identification. $\mathbf{R}_S(\mathbb{Q}) = \mathbf{E}_{H \sim \mathbb{Q}}[\mathbf{R}_S(H)]$ is also defined -- the observation of $S$ is used, not integrated out.

Finally, we also use distribution divergences based on the $\mathcal{H}$-divergence proposed by \citet{ben2007analysis}. This divergence is a specification of the $\mathcal{A}$-distance \citep{kifer2004detecting} which relaxes the total variation distance by considering only a subset $\mathcal{A}$ of measurable sets when taking the supremum. In particular, the $\mathcal{H}$-divergence considers sets identifiable by a class $\mathcal{H} \subseteq \{0,1\}^\mathcal{X}$
\begin{equation}\small
\label{eqn:hyp_class_divergence}
    \mathbf{d}_\mathcal{H}(\mathbb{D}_1, \mathbb{D}_2) \defn \sup\nolimits_{h \in \mathcal{H}} \big \lvert \mathbf{E}[h(X_1)] - \mathbf{E}[h(X_2)] \big \rvert
\end{equation}
where $X_i \sim \mathbb{D}_i$. 
While it is typically defined with a factor of 2, we omit this for convenience. Given a class $\mathcal{H} \subseteq \mathcal{Y}^\mathcal{X}$, we first study the $\mathcal{H}\Delta\mathcal{H}$-divergence based on the class 
\begin{equation}\small
\label{eqn:mid_class_defn}
\begin{split}
\mathcal{H}\Delta\mathcal{H} & \defn \big \{x \mapsto 1-\mathbf{1}_{\{h(x)\}}\{h'(x)\} \mid (h,h') \in \mathcal{H}^2 \big \}.  
\end{split}
\end{equation}
This is a multiclass generalization, which simplifies to the original binary definition of Ben-David et al. when $|\mathcal{Y}| = 2$.
\subsection{Some Existing Adaptation Bounds}
\label{sec:background_bounds}
In this section, we discuss two adaptation bounds. More detailed knowledge of these bounds will be useful later for comparison with our proposed bounds. First, we discuss the seminal uniform convergence bound proposed by \citet{ben2007analysis, ben2010theory}. Second, we discuss a PAC-Bayesian bound proposed by \citet{germain2020pac}.
\subsubsection{Adaptation Based on Uniform Convergence}
\begin{theorem}\label{thm:ben2010theory}
\citep{ben2010theory}
Let $\mathcal{Y}$ be binary. For all $\delta > 0$, w.p. at least $1-\delta$, for all $h \in \mathcal{H}$
\begin{equation}\small
\begin{split}
    & \mathbf{R}_\mathbb{T}(h) \leq \lambda + \mathbf{R}_S(h) + \mathbf{d}_{\mathcal{H}\Delta\mathcal{H}}(S_X, T_X) \\
    & + 4\sqrt{\tfrac{4 \nu \ln (2m) - \ln (\delta / 4)}{m}} + 2\sqrt{\tfrac{8\nu \ln(em / \nu) - 2 \ln(\delta / 4)}{m}}
\end{split}
\end{equation}
where $\lambda = \min_{\eta \in \mathcal{H}} \mathbf{R}_\mathbb{S}(\eta) + \mathbf{R}_\mathbb{T}(\eta)$ and $\nu = \mathrm{VCDim}(\mathcal{H})$.
\end{theorem}
The seminal result above is the standard adaptation bound on which many newer results are based. Still, this uniform convergence bound is not well-suited for every application. We discuss some limitations below.

\paragraph{Uniform Sample Complexity}
Simply put, uniform convergence is too conservative: it assigns the same sample complexity to each outcome of our learning algorithm, regardless of the solution quality. As discussed in Section~\ref{sec:intro}, this prevents us from studying important properties of a model as a function of its sample complexity.

\paragraph{Model-Independent Divergence}
In general, divergence is meant to characterize the similarity in feature distributions under the source $\mathbb{S}$ and the target $\mathbb{T}$. Similar to above, independence of the divergence $\mathbf{d}_{\mathcal{H}\Delta\mathcal{H}}$ and the model $h$ is overly conservative and makes this term insensitive to changes in the outcome of our learning algorithm. For example, when $\mathcal{H}$ is fixed, this divergence cannot distinguish between a random initialization and a carefully trained solution.

\paragraph{Sample-Independent Adaptability}
The term $\lambda$ is often called the \textit{adaptability}. It is a measure of similarity in the labeling functions of $\mathbb{S}$ and $\mathbb{T}$, characterizing the extent to which one hypothesis in $\mathcal{H}$ can do well on \textit{both} of these distributions. When no such hypothesis exists, it is unclear how a learner could successfully adapt by minimizing risk on the source distribution \citep{ben2010theory}. Importantly, this term has been central to the theoretical discussion of failure-cases in widely used DA algorithms \citep{johansson2019support, zhao2019learning}. Meanwhile, estimation of $\lambda$ remains an under-studied research area \citep{redko2020ASO}. One problem, which we observe, is independence of $\lambda$ from the samples $S$ and $T$. In particular, one cannot directly compute the population statistic $\lambda$ in typical circumstances. Instead, one might estimate using $\min_\eta \mathbf{R}_S(\eta) + \mathbf{R}_T(\eta)$, but this requires verifying generalization of a learned model $h^* \in \argmin_\eta \mathbf{R}_S(\eta) + \mathbf{R}_T(\eta)$ using a holdout set or some other descriptor of generalization performance (e.g., such as a learning bound). This is undesirable when, as in this paper, we wish to study adaptability in an empirical context. As we show in later experiments (Section~\ref{sec:experiments}), the extra generalization requirement typically inflates our estimation of $\lambda$, and subsequently, mars the results we would like to interpret.

\paragraph{Binary Label Space}
It is also important to note that this bound was designed for binary learners. Computation of the $\mathcal{H}\Delta\mathcal{H}$-divergence is the most concerning issue, since existing algorithms for computation rely on symmetry of $\mathcal{H}$ and ERM over the class $\mathcal{H}\Delta\mathcal{H}$. In Section~\ref{sec:method_approx}, we discuss these issues in detail and present some solutions.
\subsubsection{A PAC-Bayesian Bound for Binary Learners}
We give Thm.~\ref{thm:germain2020pac} in Appendix~\ref{sec:germain2020pac}, which is one of the first PAC-Bayesian adaptation bounds. While \citet{germain2013pac, germain2016new, germain2020pac} propose other bounds, we focus on Thm.~\ref{thm:germain2020pac} because it is easiest to compare to the proposal of \citet{ben2010theory}. 
While tailored to Thm.~\ref{thm:germain2020pac}, the weaknesses discussed below are generally applicable to other bounds of Germain et al. 

\paragraph{Benefits Compared to Thm.~\ref{thm:ben2010theory}}
One benefit of Thm.~\ref{thm:germain2020pac} is that the divergence employed in this bound is \textbf{model-dependent} (rather than independent); namely, it depends on the Gibbs predictor $\mathbb{Q}$, whose target error we bound. As mentioned, model-independence is an overly conservative quality and \citet{germain2020pac} show this formally by proving their divergence actually lowerbounds $\mathbf{d}_{\mathcal{H}\Delta\mathcal{H}}$ for all $\mathbb{Q}$ and $\mathcal{H}$. Another primary benefit is that Thm.~\ref{thm:germain2020pac} employs a \textbf{non-uniform} sample complexity. Specifically, complexity is measured through a KL-divergence $\mathrm{KL}(\mathbb{Q} \mid \mid \mathbb{P})$, which explicitly depends on the outcome of the learning algorithm $\mathbb{Q}$. Simply put, a model is complex if it deviates much from our prior knowledge, which is captured in the prior $\mathbb{P}$.

\paragraph{Weaknesses Shared with Thm.~\ref{thm:ben2010theory}}
Despite its benefits over Thm.~\ref{thm:ben2010theory}, Thm.~\ref{thm:germain2020pac} also shares some weaknesses. First, the proposed adaptability term is also \textbf{sample-independent}. Second, the bound is still designed for a \textbf{binary label space} $\mathcal{Y}$. Unlike the case of Thm.~\ref{thm:ben2010theory}, it is not computation of the bound that causes concern, but the \textit{validity} of the bound in mutliclass settings. In particular, the problem arises because the proof of Thm.~\ref{thm:germain2020pac} relies on a decomposition of the risk which assumes $|\mathcal{Y}| = 2$. This decomposition does not hold, in general, when $\mathcal{Y}$ is larger. In fact, \citet{germain2020pac}, themselves, observe Thm.~\ref{thm:germain2020pac} is not easily extended to multiclass settings, leaving the investigation of such PAC-Bayes bounds as an open problem. For some additional empirical study of Thm.~\ref{thm:germain2020pac}, see Appendix~\ref{sec:comp2germain}.
\subsection{Other Related Works}
\label{sec:related}
Besides those works discussed above, there are some additional works, which propose alternate theories of adaptation. Some theories of adaptation use distinct integral probability metrics in place of the $\mathcal{H}$-divergence \citep{redko2017theoretical, shen2018wasserstein, johansson2019support}, while others have sought to generalize and modify the $\mathcal{H}$-divergence \citep{mansour2009domain, kuroki2019unsupervised, zhang2019bridging}. Meanwhile, others focus on assumptions distinct from small adaptability. These include covariate shift \citep{sugiyama2007covariate, you2019towards}, label shift \citep{lipton2018detecting} and \textit{generalized} label shift \citep{tachet2020domain}. The DA problem can also be modeled through causal graphs \citep{zhang2015multi, magliacane2018domain} and some extensions to DA consider a meta-distribution over targets \citep{blanchard2021domain, albuquerque2019adversarial, deng2020representation}. Notably, most assumptions are untestable in practice and not many works consider such testing, even in controlled research settings where it might be possible. As we are aware, we are the first to use a sample-dependent adaptability, which improves estimation in empirical study.

In adaptation, PAC-Bayesian results are almost exclusively due to \citet{germain2013pac, germain2016new, germain2020pac}. Albeit, in transfer learning some work does exist \citep{li2007bayesian, mcnamara2017risk}. Most directly, our work employs the PAC-Bayes bound of \citet{maurer2004note} in proofs as well as some techniques of \citet{langford2001not, dziugaite2017computing}, and \citet{perez2021tighter} in empirical study. Most notably, ours is the only PAC-Bayesian work to propose multiclass adaptation bounds. A more in depth coverage of relevant literature -- for adaptation and PAC-Bayes -- is available in Appendix~\ref{sec:ext_related}.

\section{Proposed Bounds}
\label{sec:method}
In this section, we give the proposed adaptation bounds for multiclass learners. We also provide novel algorithms for computing two multiclass divergence terms and compare these to existing approaches. Lastly, we give a second adaptation bound which removes the need for inefficient Monte-Carlo estimation of divergences dependent on a Gibbs predictor. Proof of all results is given in Appendix~\ref{sec:proofs}.
\subsection{A PAC-Bayesian Adaptation Bound for Multiclass Learners}
\label{sec:method_bound}
We begin by introducing a model-dependent class of hypotheses similar to $\mathcal{H}\Delta\mathcal{H}$. Precisely, for $h \in \mathcal{H}$,
\begin{equation}\small
    h\Delta\mathcal{H} \defn \big \{x \mapsto 1-\mathbf{1}_{\{h(x)\}}\{h'(x)\} \mid h' \in \mathcal{H} \big \}.  
\end{equation}
With it, we propose to use the $h\Delta\mathcal{H}$-divergence $\mathbf{d}_{h\Delta\mathcal{H}}$ in our adaptation bounds. This divergence is a \textbf{model-dependent} extension of the $\mathcal{H}\Delta\mathcal{H}$-divergence, applicable in multiclass settings. It is easy to observe from the definitions that this new divergence lowerbounds the $\mathcal{H}\Delta\mathcal{H}$-divergence for all $\mathcal{H}$ and all $h$. Both \citet{zhang2019bridging} and \citet{kuroki2019unsupervised} study similar divergences for bounding 01-loss in the binary case, but we are first to use this divergence with non-uniform sample complexity, and also, the first to use this divergence in an adaptation bound for multiclass learners.\footnote{We discuss a multiclass proposal of \citet{zhang2019bridging} later, but it is based on margin loss and used for uniform convergence.} Our full proposal requires novel technique and theoretical study to compute this divergence (see Section~\ref{sec:method_approx}).

Next, we give the proposed adaptation bound. As alluded, the bound has a number of notable features and we expand on these in comparison to Thms.~\ref{thm:ben2010theory} and \ref{thm:germain2020pac} after its statement.
\begin{theorem}
\label{thm:pb-bound}
For any $\mathbb{P}$ over $\mathcal{H}$, all $\delta > 0$, w.p. at least $1-\delta$, for all $\mathbb{Q}$ over $\mathcal{H}$
\begin{equation}\small
\begin{split}
    \mathbf{R}_\mathbb{T}(\mathbb{Q}) & \leq \tilde{\lambda}_{S, T} + \mathbf{R}_S(\mathbb{Q}) + \mathbf{E}_{H \sim \mathbb{Q}}[\mathbf{d}_{\mathcal{C}_{H}}(S_X, T_X)] \\
    & + \sqrt{\tfrac{\mathrm{KL}(\mathbb{Q} \mid \mid \mathbb{P}) + \ln \sqrt{4m} - \ln ( \delta ) }{2m}}
\end{split}
\end{equation}
where $\tilde{\lambda}_{S,T} = \min_{\eta \in \mathcal{H}} \mathbf{R}_S(\eta) + \mathbf{R}_T(\eta)$ and we may choose either $\mathcal{C}_h = \mathcal{H} \Delta \mathcal{H}$ for all $h$ as before or $\mathcal{C}_h = h \Delta \mathcal{H}$.
\end{theorem}
\paragraph{Comparison to Thms.~\ref{thm:ben2010theory} and \ref{thm:germain2020pac}}
By design, the bound proposed above resolves the weaknesses mentioned in Section~\ref{sec:background_bounds}. First and foremost, we remove the requirement that $\mathcal{Y}$ is binary. Second, we use a \textbf{non-uniform} notion of sample complexity; i.e., $\mathrm{KL}(\mathbb{Q}\mid\mid\mathbb{P})$. Third, Thm.~\ref{thm:pb-bound} allows for \textit{either} a \textbf{model-dependent} or \textbf{model-independent} notion of data-distribution divergence. While model-independent divergences do have some weaknesses, we retain them in our bound since, as discussed later, they can be more efficient. 
Lastly, Thm.~\ref{thm:pb-bound} employs a \textbf{sample-dependent} notion of adaptability. Compared to $\lambda$, the new adaptability $\tilde{\lambda}_{S, T}$ is the smallest error achievable on the \textit{samples}. In research contexts wherein we assume access to target labels for purpose of studying our assumptions, we later show that this quantity is fairly easy to empirically bound. 
\subsection{Approximating Multiclass Divergence}
\label{sec:method_approx}
\subsubsection{The Multiclass \texorpdfstring{$\mathcal{H}\Delta\mathcal{H}$}{H{Delta}H}-divergence}
\label{sec:method_approx_mid_div} First, since the $\mathcal{H}\Delta\mathcal{H}$-divergence is model-independent, the expectation with respect to the Gibbs predictor $\mathbb{Q}$ simplifies significantly. In particular, if $\mathcal{C}_h = \mathcal{H}\Delta\mathcal{H}$, we have
\begin{equation}\small
\mathbf{E}_{H \sim \mathbb{Q}}[\mathbf{d}_{\mathcal{C}_H}(S_X, T_X)] = \mathbf{d}_{\mathcal{H}\Delta\mathcal{H}}(S_X, T_X).
\end{equation}
Thus, computation of this divergence simplifies to computing the $\mathcal{H}\Delta\mathcal{H}$-divergence for multiclass learners. 

\paragraph{Summary of Approach}
In general, we take inspiration from the proposal of \citet{ben2010theory} who compute $\mathcal{H}\Delta\mathcal{H}$-divergence when models in $\mathcal{H}$ have binary output. Namely, we frame computation as minimization of error in a specific classification problem. To adapt this strategy to the multiclass setting, we do two primary things. First, we remove the assumption that $\mathcal{H}$ is symmetric. This is important for multiclass settings since we have no reason to believe $\mathcal{H}\Delta\mathcal{H}$ is typically symmetric.
We replace the symmetry in $\mathcal{H}\Delta\mathcal{H}$ with a symmetry in our classification problems. Second, for score-based classifiers such as neural networks, we give an optimization procedure for approximating ERM over this class based on a surrogate loss function. As we are aware, this is the first algorithm for approximating ERM over $\mathcal{H}\Delta\mathcal{H}$ when models in $\mathcal{H}$ have multiclass output.

\paragraph{Reduction to ERM}
\begin{theorem}
\label{thm:mid_div_red2erm}
Let $\mathcal{C} = \mathcal{H}\Delta\mathcal{H}$. Almost surely,
\begin{equation}\small
    \mathbf{d}_{\mathcal{C}}(S_X, T_X) = \max \begin{rcases}
    \begin{dcases}
       1 - \min_ {\varphi \in \mathcal{C}}
       \mathbf{R}_P(\varphi) + \mathbf{R}_Q(\varphi), \\
      1 - \min_ {\varphi \in \mathcal{C}} \mathbf{R}_U(\varphi) + \mathbf{R}_V(\varphi)
    \end{dcases}
  \end{rcases}
\end{equation}
where
\begin{equation}\small
\begin{split}
    P & = \big((X_i, 1) \mid X_i \in S_X \big), \ Q = \big((\tilde{X}_i, 0) \mid \tilde{X}_i \in T_X \big), \\
    U & = \big((X_i, 0) \mid X_i \in S_X \big), \ V = \big ((\tilde{X}_i, 1 ) \mid \tilde{X}_i \in T_X \big).
\end{split}
\end{equation}
\end{theorem}
Notice, pooled samples $P +_\mathrm{c} Q$ and $U +_\mathrm{c} V$ define binary classification problems ($+_\mathrm{c}$ is concatenation). Namely, they represent an identification problem wherein the learner must distinguish between the samples $S_X$ and $T_X$. To compute divergence as above, we need only select $\varphi$ to minimize the sum of class-conditional error rates for these problems. Even in simple cases, risk-minimization can be computationally hard \citep{shalev2014understanding}. Thus, we instead select $\varphi$ by optimizing a surrogate loss.

\paragraph{Approximate Minimization via Surrogate}
WLOG, $\mathcal{Y} = \{1, \ldots, C\}$. We consider a score-based class $\mathcal{S}$ written
\begin{equation}\small
    \label{eqn:mcsb_hclass}
    \mathcal{S} \defn \big \{ \Psi_\mathbf{f} \mid \mathbf{f} \in \mathcal{F} \big\}; \quad \Psi_\mathbf{f}(x) \defn  \argmax\nolimits_{\ell \in [C]} \mathbf{f}_\ell(x)
\end{equation}
with $\mathcal{F} \subseteq \{\mathbf{f} \mid \mathbf{f}_\ell : \mathcal{X} \to \mathbb{R}, \ \ell \in [C]\}$ a class of scoring-functions. In case of ties, suppose $\argmax$ returns the least label. Using the na\"ive definition in Eq.~\eqref{eqn:mid_class_defn},
\begin{equation}\small
    \mathcal{S}\Delta\mathcal{S} \defn \big \{x \mapsto 1-\mathbf{1}_{\{\Psi_\mathbf{f}(x)\}}\{\Psi_\mathbf{g}(x)\} \mid (\mathbf{f},\mathbf{g}) \in \mathcal{F}^2 \big \}.
\end{equation}
At first glance, it is unclear how to pick $\varphi \in \mathcal{S}\Delta\mathcal{S}$ to minimize error on a given dataset. So, in place of this obscure definition, the following result gives a surrogate loss which upperbounds the 01-loss on the original problem. Thus, we indirectly reduce the error by minimizing the surrogate.
\begin{theorem}
\label{thm:surrogate_loss}
Suppose $\tau : \mathbb{R} \to \mathbb{R}_{\geq 0}$ is differentiable and monotone increasing. Let $\mathbf{A} = \big (\tau \circ \mathbf{g}(x) \big) \cdot \big (\tau \circ \mathbf{f}(x)^\mathrm{T} \big )$ with $\tau$ applied element-wise and $\mathbf{f}, \mathbf{g} \in \mathcal{F}$. Set
\begin{equation}\small
\begin{split}
    z(x) & \defn \max\nolimits_{(j,k) \in [C]^2} \mathbf{A}_{jk} - \max\nolimits_{i \in [C]} \mathbf{A}_{ii}, \\
    \mathcal{L}(z, y) & \defn \ln(1 + \exp(-(2y-1) \cdot z)) / \ln(2).
\end{split}
\end{equation}
Then, if $\varphi(x) = 1-\mathbf{1}_{\{\Psi_\mathbf{f}(x)\}}\{\Psi_\mathbf{g}(x)\}$, we have
\begin{equation}\label{eqn:score-based-ub}\small
    \mathbf{R}_\mathbb{D}(\varphi) \leq \mathbf{E}_{(X,Y) \sim \mathbb{D}}  \ \mathcal{L}(z(X), Y)
\end{equation}
for any distribution $\mathbb{D}$ s.t. $\mathbf{f}(X)$ has no repeated scores and $\mathbf{g}(X)$ has no repeated scores, almost surely.\footnote{This stipulation on $\mathbb{D}$ is not too strict. It only assumes ties in the scores of $\mathbf{f}$ or $\mathbf{g}$ are \textit{very} unlikely, so these ties can be ignored.}
\end{theorem}
We point out the log loss $\mathcal{L}(z, y)$ is differentiable with respect to $z$ and $z(x)$ is differentiable with respect to $\mathbf{f}$ and $\mathbf{g}$. In practice, functions in $\mathcal{F}$ -- such as $\mathbf{f}$ and $\mathbf{g}$ -- are typically differentiable with respect to a real-parameter vector, which also defines the function. For example, this is precisely the case for neural networks. In these contexts, since composition preserves differentiability, the output of the surrogate $\mathcal{L}$ is differentiable with respect to the real-parameter vector. So, the RHS of Eq.~\eqref{eqn:score-based-ub} may be approximately minimized using batch SGD. At this point, the proposed algorithm should be familiar to the typical practitioner. It is equivalent to the manner in which we usually optimize a neural network, except for  the new surrogate $(\mathbf{f}, \mathbf{g}, x, y) \mapsto \mathcal{L}(z(x), y)$. 
\subsubsection{The Multiclass \texorpdfstring{$h\Delta\mathcal{H}$}{h{Delta}H}-divergence}
\label{sec:method_approx_mdp_div}
When $\mathcal{C}_h = h\Delta\mathcal{H}$ the divergence term is model-dependent and the expectation with respect to the Gibbs predictor $\mathbb{Q}$ becomes a challenge. For neural networks, even the Gibbs risk $\mathbf{R}_S(\mathbb{Q})$ does not have a known analytic solution. Instead, it is common to approximate using Monte-Carlo sampling \citep{langford2001not, dziugaite2017computing, dziugaite2021role, perez2021tighter}. By Hoeffding's Inequality, w.p. at least $1-\delta$, we approximate
\begin{equation}\label{eqn:montecarlo}\small
    \underset{H \sim \mathbb{Q}}{\mathbf{E}}[\mathbf{d}_{\mathcal{C}_H}(S_X, T_X)] \leq \frac{1}{k}\sum_{i=1}^k \mathbf{d}_{\mathcal{C}_{H_i}}(S_X, T_X) + \sqrt{\tfrac{\ln 2 / \delta}{2k}}
\end{equation}
where $(H_i)_{i=1}^k \sim \mathbb{Q}^k$. Using the RHS as an approximation, our computation reduces to computing $\mathbf{d}_{\mathcal{C}_h}$ for any deterministic $h \in \mathcal{H}$. Upon sampling from $\mathbb{Q}$, we can apply the algorithm for computing $\mathbf{d}_{\mathcal{C}_h}$ to each point in the sample. In light of this, the next part focuses on computing $\mathbf{d}_{\mathcal{C}_h}$ for deterministic $h$. We proceed as before, reducing computation to risk-minimization for a specific classification problem.

\textbf{Reduction to ERM with Constrained Labeling Function}
\begin{theorem}
\label{thm:mdp_div_red2erm}
Almost surely, for all $h \in \mathcal{H}$
\begin{equation}\small
\mathbf{d}_{\mathcal{C}_h}(S_X, T_X) =  \max \begin{rcases}
    \begin{dcases}
       1 - \underset{\bar{h} \in \Upsilon}{\min_ {\varphi \in \mathcal{H},}}
       \mathbf{R}_{P}(\varphi) + \mathbf{R}_{Q}(\varphi), \\
      1 - \underset{\bar{h} \in \Upsilon}{\min_ {\varphi \in \mathcal{H},}} \mathbf{R}_{U}(\varphi) + \mathbf{R}_{V}(\varphi)
    \end{dcases}
  \end{rcases}
\end{equation}
where $\mathcal{C}_h = h\Delta\mathcal{H}$ and
\begin{equation}\small
\begin{split}
    P({\bar{h}}) & = \big ((X_i, \bar{h}(X_i)) \mid X_i \in S_X \big ), \\
    Q & = \big ((\tilde{X}_i, h(\tilde{X}_i)) \mid \tilde{X}_i \in T_X \big ), \\
    U & = \big ((X_i, h(X_i)) \mid X_i \in S_X \big),\\
    V({\bar{h}}) & = \big((\tilde{X}_i, \bar{h}(\tilde{X}_i)) \mid \tilde{X}_i \in T_X \big)
\end{split}
\end{equation}
and $\Upsilon = \{\bar{h} \in \mathcal{Y}^\mathcal{X} \mid \bar{h}(x) \neq h(x), \ \forall x \in \mathcal{X}\}$.
\end{theorem}
As before, the result describes two classification problems. This time, the learner's goal is to agree with $h$ on one sample, while disagreeing with $h$ (in the way specified by $\bar{h}$) on the other sample. We minimize the class-conditional error rates by selecting from the class $\mathcal{H}$ used for the original prediction task, rather than $\mathcal{H}\Delta\mathcal{H}$. So, we make the obvious proposal: re-use whichever approximation technique we used to select $h$ in the first place. In our experiments, since $\mathcal{H}$ is a space of neural networks, we use batch SGD on an NLL loss.

\paragraph{A Heuristic for Selecting from $\Upsilon$}
Reduction to ERM in the multiclass setting also requires specification of $\bar{h} \in \Upsilon$. Specifically, $\bar{h}$ should aid in minimizing the class-conditional error rates. In our experiments, we found a simple strategy to be fairly effective. Namely, we specify $\bar{h}$ by always picking the label with the second-highest confidence according to $h$. So, $\bar{h}$ disagrees with $h$ on all of $\mathcal{X}$, but does so in the ``most reasonable'' way according to the probabilities output by $h$. This approach uses the probabilities output by $h$ to rank similarity of labeling functions and supposes the most ``similiar'' labeling function in $\Upsilon$ will be easiest for another hypothesis in $\mathcal{H}$ to simultaneously learn. Mathematically, our solution satisfies
\begin{equation}\small
\bar{h} \in \argmax\nolimits_{\upsilon \in \Upsilon} \sum\nolimits_{x \in \mathcal{X}} h_{\upsilon(x)}(x)
\end{equation}
where $h_\ell(x)$ is the score assigned to the label $\ell$. This heuristic is a practical solution that avoids search over $\Upsilon$, which will typically be unknown, unless we inefficiently enumerate using membership constraints. As we are aware, there is no known algorithm to efficiently select minimizers from $\Upsilon$ and $\mathcal{H}$, simultaneously, as called for by Thm.~\ref{thm:mdp_div_red2erm}. Besides the heuristic, we leave this problem as future work.

\paragraph{Comparison to Some Related Approaches}
Considering a binary label space $\mathcal{Y}$, \citet{kuroki2019unsupervised} propose a similar algorithm. The multiclass setting we consider does require some differences, primarily, related to the distinct degrees of freedom in multiclass and binary classification. First, our proposal removes the requirement that $\mathcal{\mathcal{H}}$ is symmetric since, as we are aware, this concept is not well-defined for hypotheses with multiclass output. Similar to before, we replace the symmetry required of $\mathcal{H}$ with symmetry in our classification problems. Second, in multiclass settings, the reduction strategy necessitates a new parameter to optimize: the labeling function $\bar{h}$. Besides our proposed heuristic for optimization, proof of this fact is \textit{not} a straightforward extension of the work of \citet{kuroki2019unsupervised}. In fact, it requires a different proof-technique (see Appendix~\ref{sec:mdp_div_red2erm}). In a multiclass setting, \citet{zhang2019bridging} also propose an approach for approximating a mutliclass divergence. In general, our two techniques require different consideration because their divergence is based on a margin loss, rather than 01-loss. Notably, the multiclass bounds of \citet{zhang2019bridging} use uniform sample complexity, unlike our proposed non-uniform approach. Further, working directly with 01-loss, as we do, avoids
any loosening of the bound via the margin penalty. 
\subsection{Efficiency Through Flat-Minima}
\label{sec:method_efficiency}
In full view of Section~\ref{sec:method_approx_mdp_div}, the reader may rightfully be concerned about the efficiency of the proposed technique for approximating $\mathbf{E}_H[\mathbf{d}_{\mathcal{C}_H}(\cdot, \cdot)]$. In particular, the suggestion requires training $k$ distinct neural networks: one for each $H_i \sim \mathbb{Q}$. Typically, $k$ will be large -- e.g., larger than 100 -- to control the size of the upperbound in Eq.~\eqref{eqn:montecarlo}, so this is not computationally feasible for practical applications. This problem is not totally unique to the model-dependent $h\Delta\mathcal{H}$-divergence, either. Common invariant feature-learning algorithms -- e.g., DANN \citep{ganin2015unsupervised} -- actually modify the feature distribution over which the classifier learns. In these cases, even the $\mathcal{H}\Delta\mathcal{H}$-divergence becomes dependent on the model \citep{johansson2019support}. In particular, supposing every model $h \in \mathcal{H}$ is the composition $c_h \circ f_h$ of a classifier $c_h$ and a feature extractor $f_h$, the modified $\mathcal{H}\Delta\mathcal{H}$-divergence results from the following restriction
\begin{equation}\small
    [\mathcal{H}\Delta\mathcal{H}]_h \defn \{1 - \mathbf{1}_{\{c_p\circ f_h(\cdot)\}}\{c_q \circ f_h(\cdot)\} \mid (p,q) \in \mathcal{H}^2\}.
\end{equation}
A similar restriction can be defined for the class $h\Delta\mathcal{H}$
\begin{equation}\small
    [h\Delta\mathcal{H}]_h \defn \{1 - \mathbf{1}_{\{c_h\circ f_h(\cdot)\}}\{c_q \circ f_h(\cdot)\} \mid q \in \mathcal{H}\}.
\end{equation}
In both cases, due to the dependence on $h$, the expectation over $\mathbb{Q}$ cannot be avoided as in Section~\ref{sec:method_approx_mid_div}. To resolve this frequent issue, we propose a new adaptation bound, which relies on an assumption related to flatness of the 01-loss over a (weighted) region in parameter space defined by $\mathbb{Q}$. Flatness assumptions are not unusual in PAC-Bayes and we develop this connection next.

\subsubsection{Flat-Minima and PAC-Bayes}
\label{sec:method_efficiency_bg}
An SGD solution lies in a flat-minimum if its parameters are robust to perturbation: changing the parameters (slightly) does not change the trained network's already low error rate. To put it another way, all parameter configurations near the SGD solution have identically low error. So, flatness here is an absence of ``elevation'' in error as our model moves about some region of parameter space. 
\citet{hochreiter1997flat} first discussed ``flatness'' as it relates to neural network generalization, hypothesizing that models lying in a large flat-minimum generalize well. More recently, the idea has been validated empirically at large scale. In particular, notions of the sharpness of minima are often good empirical descriptors of an SGD-trained neural network's generalization performance \citep{jiang2019fantastic, dziugaite2020search}. The motivation for using PAC-Bayes bounds is very often based on the hypothesis that flat-minima generalize well. This is because PAC-Bayes bounds implicitly encode the existence of flat-minima \citep{neyshabur2017exploring, dziugaite2017computing}. In details, for a bound to be small for some predictor $\mathbb{Q}$, both its Gibbs risk and its KL-divergence with the prior $\mathbb{P}$ must be small. Because the prior $\mathbb{P}$ typically has some variance, we know $\mathbb{Q}$ should have variance too, or else the KL-divergence will be large. Further, the variance of $\mathbb{Q}$ ensures a region of non-zero probability away from the mean. Thus, if the Gibbs risk is also small, it is required  that models in this region away from the mean all have identically low error (i.e., form a flat-minimum). Otherwise, the Gibbs risk would be inflated by probable parameter configurations with high error as illustrated in Figure~\ref{fig:flat_help}. In this sense, PAC-Bayes bounds and flat-minima go hand-in-hand. If the former is small, we know the latter exists.\footnote{This argument fails for some pathological cases. It works best with unimodal continuous $\mathbb{Q}$; e.g., the Gaussians used in Section~\ref{sec:experiments}.} 
\begin{figure}
    \centering
    \includegraphics[width=.75\columnwidth]{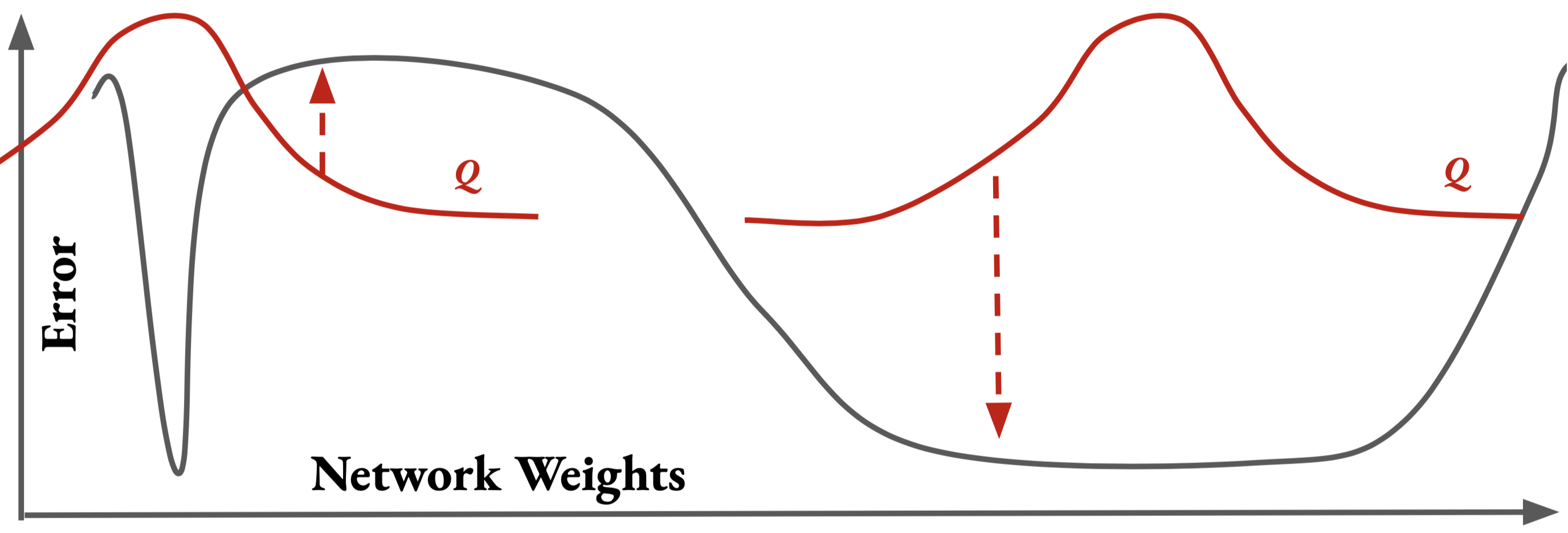}
    \caption{\small Informal illustration of flat-minimum (right) and sharp-minimum (left). For ``flat'' regions in parameter space, a unimodal Gibbs predictor $\mathbb{Q}$ with some variance has consistently low error across probable samples from $\mathbb{Q}$. Otherwise, when a region is ``sharp'', there is non-negligible likelihood of sampling a hypothesis from $\mathbb{Q}$ with high error. The expected error over $\mathbb{Q}$ is thus inflated by these likely regions of high error. }
    \label{fig:flat_help}
\end{figure}
\subsubsection{A More Efficient Adaptation Bound}
\begin{definition}\label{def:flatness} Let $\mathfrak{D}(\mathcal{H})$ be the space of distributions over $\mathcal{H}$ and fix a function $\mu: \mathfrak{D}(\mathcal{H}) \to \mathcal{H}$. The Gibbs predictor $\mathbb{Q}$ is $\rho$-flat on the distribution $\mathbb{D}$ if $\lvert \mathbf{R}_\mathbb{D}(\mathbb{Q}) - \mathbf{R}_\mathbb{D}(\mu(\mathbb{Q}))\rvert \leq \rho.$
\end{definition}
We call the function $\mu: \mathfrak{D}(\mathcal{H}) \to \mathcal{H}$ a \textit{summary function} and call the image of $\mu$ a \textit{summary}. When $\mathbb{Q}$ is implied, we typically abuse notation by writing $\mu = \mu(\mathbb{Q})$. Often, as in the definition above, we will refer to the ``flatness'' of a Gibbs predictor $\mathbb{Q}$ when it would be more precise to refer to the flatness of the (weighted) region in parameter space that this predictor defines (i.e., a around the summary $\mu$). In this sense, the above definition quantifies the flatness of a region in parameter space by the ability of the error in this region to be represented by a single hypothesis $\mu$ from that region. Intuitively, this echoes physical properties of flatness: a topographic map requires many more numbers to describe a mountainous terrain than a flat prairie. That is, each change in elevation for the mountainous terrain must be demarcated by individual numbers, while the flat prairie may only need one number to summarize the elevation. Similarly for the region around $\mu$ defined by the predictor $\mathbb{Q}$, a region is only ``flat'' if the error at $\mu$ is a good representative of the error across the whole region. Next, we give the proposed bound.
\begin{theorem}
\label{thm:pb-bound-efficient}
For any $\mathbb{P}$ over $\mathcal{H}$, all $\delta > 0$, w.p. at least $1-\delta$, for all $\mathbb{Q}$ over $\mathcal{H}$ s.t. $\mathbb{Q}$ is $\rho_S$-flat on $S$ and $\rho_T$-flat on $T$
\begin{equation}\small
\begin{split}
    \mathbf{R}_\mathbb{T}(\mathbb{Q}) & \leq \rho + \tilde{\lambda}_{S, T} + \mathbf{R}_S(\mathbb{Q}) + \mathbf{d}_{\mathcal{C}_\mu}(S_X, T_X) \\
    & + \sqrt{\tfrac{\mathrm{KL}(\mathbb{Q} \mid \mid \mathbb{P}) + \ln \sqrt{4m} - \ln ( \delta ) }{2m}}
\end{split}
\end{equation}
where $\mu$ is the summary of $\mathbb{Q}$, $\rho = \rho_S + \rho_T$, and $\mathcal{C}_\mu = \mu\Delta\mathcal{H}$.
\end{theorem}
\begin{corollary}\label{cor:pb-bound-efficient}
To study algorithms like DANN, we can instead choose $\mathcal{C}_\mu = [\mathcal{H}\Delta\mathcal{H}]_\mu$ or $\mathcal{C}_\mu = [\mu\Delta\mathcal{H}]_\mu$ in Thm.~\ref{thm:pb-bound-efficient}. The adaptability $\tilde{\lambda}_{S, T}$ is then dependent on $\mu$ as below 
\begin{equation}\small
    \small \tilde{\lambda}_{S, T}^\mu = \min\nolimits_{g \in \mathcal{H}} \Big \{ \mathbf{R}_S(c_g \circ f_\mu) + \mathbf{R}_T(c_g \circ f_\mu) \Big \}.
\end{equation}
\end{corollary}
The main bound is identical to Thm.~\ref{thm:pb-bound} except that we assume $\mathbb{Q}$ is flat on both $S$ and $T$, then use this assumption to introduce a deterministic summary $\mu$ in the divergence. This deterministic summary replaces the expectation over $\mathbb{Q}$ whose estimation was inefficient, but the new cost is inflation of the bound by $\rho$. Unfortunately, similar to adaptability terms, we cannot expect to compute $\rho$ outside of controlled research contexts, since labels are required to estimate the flatness of $\mathbb{Q}$ on $T$ (according to Def.~\ref{def:flatness}). Instead, for the bound to be practically useful, we propose to assume $\rho$ is small. This, for example, is often the suggestion when it comes to adaptability as well. Albeit, the caveats of carelessly making assumptions on adaptation problems should be noted \citep{zhao2019learning, johansson2019support}. 

We argue the assumption of small $\rho$ is not an overly strong (or careless) assumption to make. To begin with, PAC-Bayes bounds and flat-minima are already related. The only addition we make to the usual connection (see Section~\ref{sec:method_efficiency_bg}) is that flat regions \textit{remain flat} when we transfer across data distributions (or, samples). Note, we do not even require the transferred region to remain a minimum, since the size of $\rho$ is only dictated by the \textit{difference} in the Gibbs risk and the summary risk: the Gibbs risk can be high on $T$ as long as the summary risk is as well. Thus, if one is willing to accept the usual assumptions, then our additional assumption merely begs the question: \textit{Do flat regions transfer?} 

In the next section, empirically, we test this question along with the other proposals given in this text.
\section{Experiments}
\label{sec:experiments}
\begin{figure*}
    \centering
    \includegraphics[width=\linewidth]{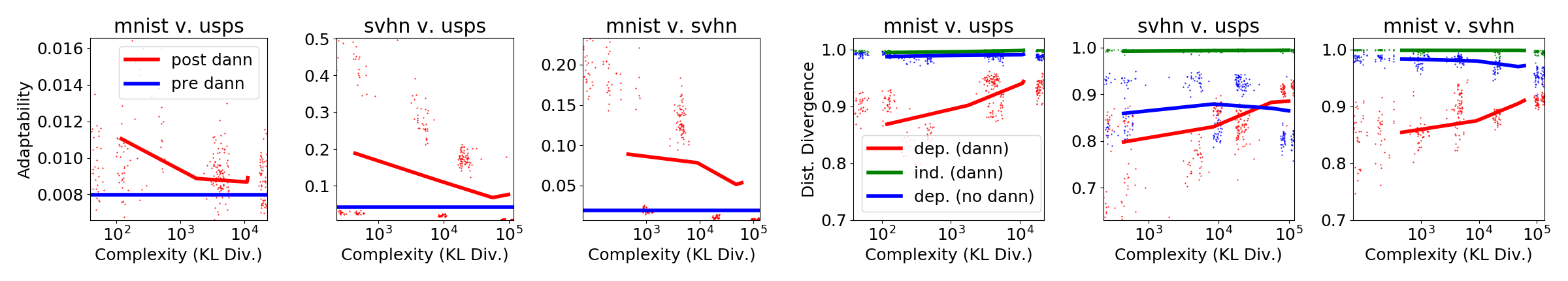}
    \caption{\small Adaptability (left) and dependent/independent divergences (right) for \textbf{DANN} on \textbf{Digits}. Solid line is median. Scatter describes unique $(S, T, \mathbb{Q})$, limited to $95\%$ or more data to filter extreme values. $\mathbb{Q}$ is a multivariate Gaussian and $\mu$ is its mean.}
    \label{fig:dann-all}
\end{figure*}
\begin{figure}
    \centering
    \includegraphics[width=0.8\columnwidth]{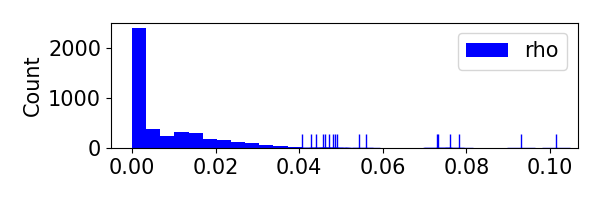}
    \caption{\small Histogram of $\rho$ estimates. Rug plot above $0.04$ displays infrequent occurrences. Each datum describes unique $(S, T, \mathbb{Q})$. $\mathbb{Q}, \mu$ are defined as in Figure~\ref{fig:dann-all}. See Appendix \ref{sec:exp_details_flatness} for details.}
    \label{fig:rho-hist}
\end{figure}
\begin{figure}
    \centering
    \includegraphics[width=.8\columnwidth]{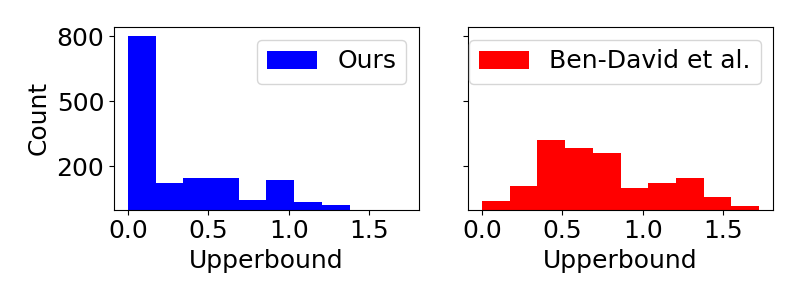}
    \caption{\small Sample-dependent (left) and independent adaptability. Each datum is for unique $(S, T, \mathcal{H})$. See Appendix~\ref{sec:exp_details_ada} for details.}
    \label{fig:lambda}
\end{figure}
\subsection{Setup}
\paragraph{Datasets}
We use a wide-array of datasets from vision and NLP: \textbf{Digits} \citep{ganin2015unsupervised}, \textbf{PACS} \citep{li2017deeper_PACS}, \textbf{Office-Home} \citep{venkateswara2017deep}, \textbf{Amazon Reviews} \citep{blitzer2007biographies}, and \textbf{Discourse} sense classification datasets \citep{prasad2008penn, ramesh2010identifying, zeyrek2020ted}. 
For \textbf{Digits}, we use the image as feature, while for \textbf{Amazon Reviews}, we use uni-gram and bi-gram features. For other datasets, we use pre-trained ResNet-50 \citep{he2016deep} or BERT \citep{devlin2018bert} features.

\paragraph{Models}
For \textbf{Digits}, we use a 4-layer CNN. For all other datasets, we use both a linear model and a 4-layer fully-connected network. For simplicity, our larger scale experiments use a simple adaptation algorithm (\textbf{SA}) which optimizes models to minimize risk on $S$. On \textbf{Digits}, we also study the \textbf{DANN} algorithm proposed by \citet{ganin2015unsupervised}, modified to train Gibbs predictors with varied regularization using \textbf{PBB} \citep{perez2021tighter}. 
We pick \textbf{Digits}, specifically, because it exhibits shift in the marginal label distributions, which can cause \textbf{DANN} to fail \citep{zhao2019learning}.
More training details are given in Appendix~\ref{sec:exp_details}.

\paragraph{Experiments}
The data points in our results are each individual experiments done on a source dataset $S$ and target dataset $T$ using a classifier $h$ or Gibbs predictor $\mathbb{Q}$. The pair $S$ and $T$ are taken from a set of data splits using the datasets discussed above (details in Appendix~\ref{sec:exp_details}). Across these splits, we consider various scenarios including: \textbf{single-source}, \textbf{multi-source}, and \textbf{within-distribution} adaptation (i.e., $\mathbb{S} = \mathbb{T}$) using multiple random data splits. On \textbf{Digits}, we also consider \textbf{natural shifts}  (i.e., noise and rotation) and \textbf{unnatural shifts} (i.e., transfer to random data). In general, we restrict the pair $(S, T$) to have a common label space. 
\subsection{Results}
\paragraph{Sample-Dependent Adaptability}
As mentioned, estimation of sample-independent adaptability (e.g., $\lambda$ in Thm.~\ref{thm:ben2010theory}) requires verification of generalization. In particular, to estimate $\lambda$ one can learn $\eta \in \mathcal{H}$ which has small sum of risks over the observed samples $S$ and $T$. In our experiments, we do so using batch SGD on a weighted NLL loss -- a common surrogate. Because $\lambda$ is a population statistic for $\mathbb{S}$ and $\mathbb{T}$, we cannot directly report the errors on $S$ and $T$ -- this is incorrect, like using training error as a validation metric. Instead, we should check the performance of $\eta$ on a heldout data subset (for example). This is the strategy we take in Figure~\ref{fig:lambda}, using Hoeffding's Inequality to produce a valid upperbound on $\lambda$. Comparably, estimating sample-dependent adaptability is much easier. By design, we \textit{can} report error on the samples $S$ and $T$ used for training $\eta$. Doing so, produces a valid upperbound:
\begin{equation}\small\label{eqn:lambda_ub}
    \forall \eta \in \mathcal{H} \ : \ \tilde{\lambda}_{S,T} \leq \mathbf{R}_S(\eta) + \mathbf{R}_T(\eta) \quad\text{(by definition).}
\end{equation}
As is visible in Figure~\ref{fig:lambda}, this strategy for estimating $\tilde{\lambda}$ is much more effective than the sample-independent strategy in revealing important information. We see from the histogram of upperbounds on $\tilde{\lambda}$ that adaptability is very often small and concentrated near 0, although this is not always the case. Comparatively, upperbounds for $\lambda$ are spread out with notable mass at large values; we miss out on the interpretation that adaptability very often is small (as we might like to assume, in practice). In the rest of our discussion, all adaptability will be sample-dependent. Note, additional experiments on adaptability are available in Appendix~\ref{sec:adaptability_boxp}.
\begin{table}\small
    \centering
    \caption{\small Correlation of $h$-dependent and -independent divergence with $|\mathbf{R}_h(S) - \mathbf{R}_h(T)|$. Columns delineate data subsets. Each datum describes unique $(S, T, h)$. See Appendix~\ref{sec:exp_details_divapprox} for details.}
    \begin{tabular}{c|c|c|c|c|c}
    &  \textbf{All} & \textbf{Digi.} & \textbf{Disc.} & \textbf{PACS+OH} & \textbf{Amaz.}\\\hline
    \textbf{model-ind.} & 0.54 & 0.15 & 0.70 & 0.41 & -0.05 \\\hline
    \textbf{model-dep.} & 0.58 & 0.23 & 0.78 & 0.14 & 0.41  \\
    \end{tabular}
    \label{tab:divapprox}
\end{table}

\paragraph{Divergence and Approximation}
In Table~\ref{tab:divapprox}, we give results for our approximation techniques applied to the model-dependent $h\Delta\mathcal{H}$-divergence and the model-independent $\mathcal{H}\Delta\mathcal{H}$-divergence. The models used in these experiments are trained using \textbf{SA}. Since there is actually no ground-truth to compare too, we report performance of our approximations on a ranking task. That is, we compare our approximations to absolute difference in risks on the source and target and compute the Spearman rank correlation. According to our adaptation bounds, smaller divergence should predict smaller difference in risk and larger divergence should predict larger difference in risk as in the ranking task we study. Any effective approximation of divergence should also mimic this behavior, allowing us to conduct an indirect evaluation. In aggregate, we observe both divergences are capable of ranking performance similarity on the source and target, which validates our approximations to some extent. For reference, a recent statistic designed for shift-detection \citep{rabanser2018failing} achieves correlation \textbf{0.29} on all data. We also observe the model-dependent divergence typically ranks ``better'' than the model-independent divergence. This, also, is to be expected according to our theory, since the model-independent divergence does not account for variation in $h$ and should thus perform worse. Overall, the nuanced agreement of our approximations with our theoretical expectations is suggestive that these techniques are effective. 

\paragraph{Do Flat Regions Transfer?}
As noted, one stipulation of practical use for Thm.~\ref{thm:pb-bound-efficient} is a small flatness value $\rho$. This is not unlike the common assumption that $\lambda$ is small and, as discussed, is related to the flat-minma hypothesis. To estimate $\rho$ and test our assumption, we select $\rho_S$ and $\rho_T$ to be the smallest values so that Def.~\ref{def:flatness} is satisfied on $S$ and $T$ using a Monte-Carlo estimate for the Gibbs Risk.\footnote{A penalty based on Hoeffding Inequality could be added to this estimate to create a valid upperbound. We do not consider this since a penalty is also added if we use the strategy in Eq.~\eqref{eqn:montecarlo}.} We train $\mathbb{Q}$ using a variant of \textbf{SA} based on the technique of \citet{perez2021tighter}. Our results indicate $\rho$ is typically small as desired with mean \textbf{0.007} and SD \textbf{0.01} across 4K+ experiments. See Figure~\ref{fig:rho-hist} for a visualization.

\paragraph{Analysis of Assumptions after DANN}
Our results in Figure~\ref{fig:dann-all} show an interesting relationship between the sample complexity of $\mathbb{Q}$ -- as measured by $\mathrm{KL}(\mathbb{Q} \mid\mid \mathbb{P})$ -- and our assumption on adaptability. Namely, we can be more confident in the assumption $\tilde{\lambda}$ is small when the sample complexity of our solution increases. A similar observation holds for the flatness term $\rho$ (see Appendix~\ref{sec:ext_dann_res} Figure~\ref{fig:rho-dann}). Our analysis suggests \textbf{DANN} may be a data-hungry algorithm, since solutions with properties we desire have large sample complexity. The practical suggestion is to use large quantities of unlabeled data when applying \textbf{DANN}, which is reasonable since unlabeled data can be ``cheap'' to acquire.

\paragraph{Analysis of Divergence after DANN}
In Figure~\ref{fig:dann-all}, according to the (more sensitive) model-dependent divergence, \textbf{DANN} reduces data-distribution divergence as it is designed to do. Still, it does not reduce divergence to the degree one might expect and, as the sample complexity of the solution increases, the gap between divergences -- before and after \textbf{DANN} -- begins to wane.
This is interesting because it shows reduction of divergence and reduction of adaptability / flatness may be competing objectives. Further, this finding echoes theoretical hypotheses in recent literature \citep{zhao2019learning, wu2019domain, johansson2019support}, while also revealing the role of sample complexity in this story. To meet our assumptions when using \textbf{DANN}, we should use large amounts of unlabeled data and allow an unconstrained solution, but to ensure \textbf{DANN} reduces distribution divergence significantly, we should instead constrain our solution to lower complexity (e.g., via regularization). Depending on problem context, there may be some optimum between these extremes, but in any case, these opposing relationships are an interesting take-away from the application of our theory.
\section{Conclusion}
\label{sec:conclusion}
In this work, we proposed the first adaptation bounds capable of studying the non-uniform sample complexity of adaptation algorithms using multiclass neural networks. Empirically, we validated the novel design-concepts in our adaptation bounds and showed our approximation techniques for some multiclass divergences were effective. In culmination, we applied our bounds to study sample complexity of a common domain-invariant learning algorithm. Our findings revealed unexpected relationships between sample complexity and important properties of the algorithm we studied. Code for reproducing our experiments is publicly available at \url{https://github.com/anthonysicilia/pacbayes-adaptation-UAI2022}.

Besides what has been done in this work, we also identify some areas of potential future work:
\paragraph{Assumptions and Heuristics}
As with previous adaptation bounds, the nature of the adaptation problem requires us to be imprecise in some cases. For one, we make a number of assumptions on adaptability and flatness. Also, our divergence computation does require some heuristics. While we study these imperfections empirically with promising results, we anticipate both shortcomings can be improved. In particular, restriction of scope to specific domains or hypothesis classes should reveal exploitable problem structure. 
\paragraph{Generalized Loss}
While we have focused on multiclass learners in this work, a PAC-Bayesian adaptation bound for general learners (e.g., with bounded loss functions) remains an open-problem. Possibly, applying our strategies to the more general framework of \citet{mansour2009domain} would be fruitful. Albeit, since algorithms for computing divergence have traditionally been loss-specific, we expect additional theoretical derivation to be required for each new loss.
\begin{acknowledgements} We thank the anonymous reviewers for helpful feedback. 

S. Hwang was supported by Institute of Information \& communications Technology Planning \& Evaluation (IITP) grant funded by the Korea government (MSIT), Artificial Intelligence Graduate Program, Yonsei University (2020-0-01361-003), and the Yonsei University Research Fund of 2022 (2022-22-0131).
\end{acknowledgements}
\bibliography{sicilia_277}

\begin{thebibliography}{115}
\providecommand{\natexlab}[1]{#1}
\providecommand{\url}[1]{\texttt{#1}}
\expandafter\ifx\csname urlstyle\endcsname\relax
  \providecommand{\doi}[1]{doi: #1}\else
  \providecommand{\doi}{doi: \begingroup \urlstyle{rm}\Url}\fi

\bibitem[Albuquerque et~al.(2020{\natexlab{a}})Albuquerque, Monteiro, Falk, and
  Mitliagkas]{albuquerque2019adversarial}
Isabela Albuquerque, Jo{\~a}o Monteiro, Tiago~H Falk, and Ioannis Mitliagkas.
\newblock Adversarial target-invariant representation learning for domain
  generalization.
\newblock \emph{arXiv:1911.00804}, 2020{\natexlab{a}}.

\bibitem[Albuquerque et~al.(2020{\natexlab{b}})Albuquerque, Monteiro, Falk, and
  Mitliagkas]{apx_albuquerque2019adversarial}
Isabela Albuquerque, Jo{\~a}o Monteiro, Tiago~H Falk, and Ioannis Mitliagkas.
\newblock Adversarial target-invariant representation learning for domain
  generalization.
\newblock \emph{arXiv:1911.00804}, 2020{\natexlab{b}}.

\bibitem[Ambroladze et~al.(2007)Ambroladze, Parrado-Hern{\'a}ndez, and
  Shawe-Taylor]{apx_ambroladze2007tighter}
Amiran Ambroladze, Emilio Parrado-Hern{\'a}ndez, and John Shawe-Taylor.
\newblock Tighter pac-bayes bounds.
\newblock \emph{NeurIPS}, 19:\penalty0 9, 2007.

\bibitem[Ben-David et~al.(2007{\natexlab{a}})Ben-David, Blitzer, Crammer,
  Pereira, et~al.]{apx_ben2007analysis}
Shai Ben-David, John Blitzer, Koby Crammer, Fernando Pereira, et~al.
\newblock Analysis of representations for domain adaptation.
\newblock \emph{NeurIPS}, 19:\penalty0 137, 2007{\natexlab{a}}.

\bibitem[Ben-David et~al.(2007{\natexlab{b}})Ben-David, Blitzer, Crammer,
  Pereira, et~al.]{ben2007analysis}
Shai Ben-David, John Blitzer, Koby Crammer, Fernando Pereira, et~al.
\newblock Analysis of representations for domain adaptation.
\newblock \emph{NeurIPS}, 19:\penalty0 137, 2007{\natexlab{b}}.

\bibitem[Ben-David et~al.(2010{\natexlab{a}})Ben-David, Blitzer, Crammer,
  Kulesza, Pereira, and Vaughan]{apx_ben2010theory}
Shai Ben-David, John Blitzer, Koby Crammer, Alex Kulesza, Fernando Pereira, and
  Jennifer~Wortman Vaughan.
\newblock A theory of learning from different domains.
\newblock \emph{Machine learning}, 79\penalty0 (1):\penalty0 151--175,
  2010{\natexlab{a}}.

\bibitem[Ben-David et~al.(2010{\natexlab{b}})Ben-David, Blitzer, Crammer,
  Kulesza, Pereira, and Vaughan]{ben2010theory}
Shai Ben-David, John Blitzer, Koby Crammer, Alex Kulesza, Fernando Pereira, and
  Jennifer~Wortman Vaughan.
\newblock A theory of learning from different domains.
\newblock \emph{Machine learning}, 79\penalty0 (1):\penalty0 151--175,
  2010{\natexlab{b}}.

\bibitem[Ben-David et~al.(2010{\natexlab{c}})Ben-David, Lu, Luu, and
  P{\'a}l]{apx_david2010impossibility}
Shai Ben-David, Tyler Lu, Teresa Luu, and D{\'a}vid P{\'a}l.
\newblock Impossibility theorems for domain adaptation.
\newblock In \emph{AISTATS}, pages 129--136. JMLR Workshop and Conference
  Proceedings, 2010{\natexlab{c}}.

\bibitem[Blanchard et~al.(2021{\natexlab{a}})Blanchard, Deshmukh, Dogan, Lee,
  and Scott]{apx_blanchard2021domain}
Gilles Blanchard, Aniket~Anand Deshmukh, {\"U}r{\"u}n Dogan, Gyemin Lee, and
  Clayton Scott.
\newblock Domain generalization by marginal transfer learning.
\newblock \emph{JMLR}, 22:\penalty0 2--1, 2021{\natexlab{a}}.

\bibitem[Blanchard et~al.(2021{\natexlab{b}})Blanchard, Deshmukh, Dogan, Lee,
  and Scott]{blanchard2021domain}
Gilles Blanchard, Aniket~Anand Deshmukh, {\"U}r{\"u}n Dogan, Gyemin Lee, and
  Clayton Scott.
\newblock Domain generalization by marginal transfer learning.
\newblock \emph{JMLR}, 22:\penalty0 2--1, 2021{\natexlab{b}}.

\bibitem[Blitzer et~al.(2007{\natexlab{a}})Blitzer, Dredze, and
  Pereira]{apx_blitzer2007biographies}
John Blitzer, Mark Dredze, and Fernando Pereira.
\newblock Biographies, bollywood, boom-boxes and blenders: Domain adaptation
  for sentiment classification.
\newblock In \emph{ACL}, pages 440--447, 2007{\natexlab{a}}.

\bibitem[Blitzer et~al.(2007{\natexlab{b}})Blitzer, Dredze, and
  Pereira]{blitzer2007biographies}
John Blitzer, Mark Dredze, and Fernando Pereira.
\newblock Biographies, bollywood, boom-boxes and blenders: Domain adaptation
  for sentiment classification.
\newblock In \emph{ACL}, pages 440--447, 2007{\natexlab{b}}.

\bibitem[Blundell et~al.(2015)Blundell, Cornebise, Kavukcuoglu, and
  Wierstra]{blundell2015weight}
Charles Blundell, Julien Cornebise, Koray Kavukcuoglu, and Daan Wierstra.
\newblock Weight uncertainty in neural network.
\newblock In \emph{ICML}, pages 1613--1622. PMLR, 2015.

\bibitem[Braud et~al.(2017)Braud, Coavoux, and S{\o}gaard]{apx_braud2017cross}
Chlo{\'e} Braud, Maximin Coavoux, and Anders S{\o}gaard.
\newblock Cross-lingual rst discourse parsing.
\newblock \emph{arXiv:1701.02946}, 2017.

\bibitem[Carlson et~al.(2003)Carlson, Marcu, and
  Okurowski]{apx_carlson2003building}
Lynn Carlson, Daniel Marcu, and Mary~Ellen Okurowski.
\newblock Building a discourse-tagged corpus in the framework of rhetorical
  structure theory.
\newblock In \emph{Current and new directions in discourse and dialogue}, pages
  85--112. Springer, 2003.

\bibitem[Catoni(2007)]{apx_catoni2007pac}
Olivier Catoni.
\newblock {PAC}-{B}ayesian supervised classification: the thermodynamics of
  statistical learning.
\newblock \emph{arXiv:0712.0248v1}, 2007.

\bibitem[Crammer et~al.(2007)Crammer, Kearns, and
  Wortman]{apx_crammer2007learning}
Koby Crammer, Michael Kearns, and Jennifer Wortman.
\newblock Learning from multiple sources.
\newblock In \emph{NeurIPS}, pages 321--328, 2007.

\bibitem[Deng et~al.(2020{\natexlab{a}})Deng, Ding, Dwork, Hong, Parmigiani,
  Patil, and Sur]{apx_deng2020representation}
Zhun Deng, Frances Ding, Cynthia Dwork, Rachel Hong, Giovanni Parmigiani,
  Prasad Patil, and Pragya Sur.
\newblock Representation via representations: Domain generalization via
  adversarially learned invariant representations.
\newblock \emph{arXiv:2006.11478}, 2020{\natexlab{a}}.

\bibitem[Deng et~al.(2020{\natexlab{b}})Deng, Ding, Dwork, Hong, Parmigiani,
  Patil, and Sur]{deng2020representation}
Zhun Deng, Frances Ding, Cynthia Dwork, Rachel Hong, Giovanni Parmigiani,
  Prasad Patil, and Pragya Sur.
\newblock Representation via representations: Domain generalization via
  adversarially learned invariant representations.
\newblock \emph{arXiv:2006.11478}, 2020{\natexlab{b}}.

\bibitem[Devlin et~al.(2019{\natexlab{a}})Devlin, Chang, Lee, and
  Toutanova]{apx_devlin2018bert}
Jacob Devlin, Ming-Wei Chang, Kenton Lee, and Kristina Toutanova.
\newblock {BERT}: Pre-training of deep bidirectional transformers for language
  understanding.
\newblock In \emph{NAACL-HLT}, pages 4171--4186. ACL, 2019{\natexlab{a}}.

\bibitem[Devlin et~al.(2019{\natexlab{b}})Devlin, Chang, Lee, and
  Toutanova]{devlin2018bert}
Jacob Devlin, Ming-Wei Chang, Kenton Lee, and Kristina Toutanova.
\newblock {BERT}: Pre-training of deep bidirectional transformers for language
  understanding.
\newblock In \emph{NAACL-HLT}, pages 4171--4186. ACL, 2019{\natexlab{b}}.

\bibitem[Dziugaite and Roy(2017{\natexlab{a}})]{apx_dziugaite2017computing}
Gintare~Karolina Dziugaite and Daniel~M Roy.
\newblock Computing nonvacuous generalization bounds for deep (stochastic)
  neural networks with many more parameters than training data.
\newblock \emph{arXiv:1703.11008}, 2017{\natexlab{a}}.

\bibitem[Dziugaite and Roy(2017{\natexlab{b}})]{dziugaite2017computing}
Gintare~Karolina Dziugaite and Daniel~M Roy.
\newblock Computing nonvacuous generalization bounds for deep (stochastic)
  neural networks with many more parameters than training data.
\newblock \emph{arXiv:1703.11008}, 2017{\natexlab{b}}.

\bibitem[Dziugaite et~al.(2020)Dziugaite, Drouin, Neal, Rajkumar, Caballero,
  Wang, Mitliagkas, and Roy]{dziugaite2020search}
Gintare~Karolina Dziugaite, Alexandre Drouin, Brady Neal, Nitarshan Rajkumar,
  Ethan Caballero, Linbo Wang, Ioannis Mitliagkas, and Daniel~M Roy.
\newblock In search of robust measures of generalization.
\newblock \emph{NeurIPS}, 33, 2020.

\bibitem[Dziugaite et~al.(2021{\natexlab{a}})Dziugaite, Hsu, Gharbieh, Arpino,
  and Roy]{apx_dziugaite2021role}
Gintare~Karolina Dziugaite, Kyle Hsu, Waseem Gharbieh, Gabriel Arpino, and
  Daniel Roy.
\newblock On the role of data in pac-bayes.
\newblock In \emph{AISTATS}, pages 604--612. PMLR, 2021{\natexlab{a}}.

\bibitem[Dziugaite et~al.(2021{\natexlab{b}})Dziugaite, Hsu, Gharbieh, Arpino,
  and Roy]{dziugaite2021role}
Gintare~Karolina Dziugaite, Kyle Hsu, Waseem Gharbieh, Gabriel Arpino, and
  Daniel Roy.
\newblock On the role of data in pac-bayes.
\newblock In \emph{AISTATS}, pages 604--612. PMLR, 2021{\natexlab{b}}.

\bibitem[Ganin and Lempitsky(2015{\natexlab{a}})]{apx_ganin2015unsupervised}
Yaroslav Ganin and Victor Lempitsky.
\newblock Unsupervised domain adaptation by backpropagation.
\newblock In \emph{ICML}, pages 1180--1189. PMLR, 2015{\natexlab{a}}.

\bibitem[Ganin and Lempitsky(2015{\natexlab{b}})]{ganin2015unsupervised}
Yaroslav Ganin and Victor Lempitsky.
\newblock Unsupervised domain adaptation by backpropagation.
\newblock In \emph{ICML}, pages 1180--1189. PMLR, 2015{\natexlab{b}}.

\bibitem[Germain et~al.(2009)Germain, Lacasse, Laviolette, and
  Marchand]{apx_germain2009pac}
Pascal Germain, Alexandre Lacasse, Fran{\c{c}}ois Laviolette, and Mario
  Marchand.
\newblock {PAC}-{B}ayesian learning of linear classifiers.
\newblock In \emph{ICML}, 2009.

\bibitem[Germain et~al.(2013)Germain, Habrard, Laviolette, and
  Morvant]{germain2013pac}
Pascal Germain, Amaury Habrard, François Laviolette, and Emilie Morvant.
\newblock A pac-bayesian approach for domain adaptation with specialization to
  linear classifiers.
\newblock In \emph{ICML}, pages 738--746. PMLR, 2013.

\bibitem[Germain et~al.(2015)Germain, Lacasse, Laviolette, March, and
  Roy]{apx_germain2015risk}
Pascal Germain, Alexandre Lacasse, Francois Laviolette, Mario March, and
  Jean-Francis Roy.
\newblock Risk {B}ounds for the {M}ajority {V}ote: From a {PAC}-{B}ayesian
  {A}nalysis to a {L}earning {A}lgorithm.
\newblock \emph{JMLR}, 16:\penalty0 787--860, 2015.

\bibitem[Germain et~al.(2016)Germain, Habrard, Laviolette, and
  Morvant]{germain2016new}
Pascal Germain, Amaury Habrard, François Laviolette, and Emilie Morvant.
\newblock A new pac-bayesian perspective on domain adaptation.
\newblock In \emph{ICML}, pages 859--868. PMLR, 2016.

\bibitem[Germain et~al.(2020{\natexlab{a}})Germain, Habrard, Laviolette, and
  Morvant]{apx_germain2020pac}
Pascal Germain, Amaury Habrard, Fran{\c{c}}ois Laviolette, and Emilie Morvant.
\newblock Pac-bayes and domain adaptation.
\newblock \emph{Neurocomputing}, 379:\penalty0 379--397, 2020{\natexlab{a}}.

\bibitem[Germain et~al.(2020{\natexlab{b}})Germain, Habrard, Laviolette, and
  Morvant]{germain2020pac}
Pascal Germain, Amaury Habrard, Fran{\c{c}}ois Laviolette, and Emilie Morvant.
\newblock Pac-bayes and domain adaptation.
\newblock \emph{Neurocomputing}, 379:\penalty0 379--397, 2020{\natexlab{b}}.

\bibitem[Gretton et~al.(2012)Gretton, Borgwardt, Rasch, Sch{\"o}lkopf, and
  Smola]{apx_gretton2012kernel}
Arthur Gretton, Karsten~M Borgwardt, Malte~J Rasch, Bernhard Sch{\"o}lkopf, and
  Alexander Smola.
\newblock A kernel two-sample test.
\newblock \emph{JMLR}, 13\penalty0 (1):\penalty0 723--773, 2012.

\bibitem[Guedj(2019)]{apx_guedj2019primer}
Benjamin Guedj.
\newblock A primer on {PAC}-{B}ayesian learning.
\newblock \emph{arXiv:1901.05353v3}, 2019.

\bibitem[He et~al.(2016{\natexlab{a}})He, Zhang, Ren, and Sun]{apx_he2016deep}
Kaiming He, Xiangyu Zhang, Shaoqing Ren, and Jian Sun.
\newblock Deep residual learning for image recognition.
\newblock In \emph{CVPR}, 2016{\natexlab{a}}.

\bibitem[He et~al.(2016{\natexlab{b}})He, Zhang, Ren, and Sun]{he2016deep}
Kaiming He, Xiangyu Zhang, Shaoqing Ren, and Jian Sun.
\newblock Deep residual learning for image recognition.
\newblock In \emph{CVPR}, 2016{\natexlab{b}}.

\bibitem[Hochreiter and Schmidhuber(1997)]{hochreiter1997flat}
Sepp Hochreiter and J{\"u}rgen Schmidhuber.
\newblock Flat minima.
\newblock \emph{Neural computation}, 9\penalty0 (1):\penalty0 1--42, 1997.

\bibitem[{Hull}(1994)]{apx_uspsdataset}
J.~J. {Hull}.
\newblock A database for handwritten text recognition research.
\newblock \emph{IEEE Transactions on Pattern Analysis and Machine
  Intelligence}, 16\penalty0 (5):\penalty0 550--554, 1994.
\newblock \doi{10.1109/34.291440}.

\bibitem[Jiang et~al.(2019)Jiang, Neyshabur, Mobahi, Krishnan, and
  Bengio]{jiang2019fantastic}
Yiding Jiang, Behnam Neyshabur, Hossein Mobahi, Dilip Krishnan, and Samy
  Bengio.
\newblock Fantastic generalization measures and where to find them.
\newblock In \emph{ICLR}, 2019.

\bibitem[Johansson et~al.(2019{\natexlab{a}})Johansson, Sontag, and
  Ranganath]{apx_johansson2019support}
Fredrik~D Johansson, David Sontag, and Rajesh Ranganath.
\newblock Support and invertibility in domain-invariant representations.
\newblock In \emph{AISTATS}, pages 527--536. PMLR, 2019{\natexlab{a}}.

\bibitem[Johansson et~al.(2019{\natexlab{b}})Johansson, Sontag, and
  Ranganath]{johansson2019support}
Fredrik~D Johansson, David Sontag, and Rajesh Ranganath.
\newblock Support and invertibility in domain-invariant representations.
\newblock In \emph{AISTATS}, pages 527--536. PMLR, 2019{\natexlab{b}}.

\bibitem[Keskar et~al.(2017)Keskar, Nocedal, Tang, Mudigere, and
  Smelyanskiy]{keskar2017large}
Nitish~Shirish Keskar, Jorge Nocedal, Ping Tak~Peter Tang, Dheevatsa Mudigere,
  and Mikhail Smelyanskiy.
\newblock On large-batch training for deep learning: Generalization gap and
  sharp minima.
\newblock In \emph{ICLR 2017}, 2017.

\bibitem[Kifer et~al.(2004)Kifer, Ben-David, and Gehrke]{kifer2004detecting}
Daniel Kifer, Shai Ben-David, and Johannes Gehrke.
\newblock Detecting change in data streams.
\newblock In \emph{VLDB}, volume~4, pages 180--191, 2004.

\bibitem[Kishimoto et~al.(2020)Kishimoto, Murawaki, and
  Kurohashi]{apx_kishimoto-etal-2020-adapting}
Yudai Kishimoto, Yugo Murawaki, and Sadao Kurohashi.
\newblock Adapting {BERT} to implicit discourse relation classification with a
  focus on discourse connectives.
\newblock In \emph{LREC}, pages 1152--1158, Marseille, France, May 2020.
  European Language Resources Association.
\newblock ISBN 979-10-95546-34-4.
\newblock URL \url{https://aclanthology.org/2020.lrec-1.145}.

\bibitem[Kuroki et~al.(2019{\natexlab{a}})Kuroki, Charoenphakdee, Bao, Honda,
  Sato, and Sugiyama]{apx_kuroki2019unsupervised}
Seiichi Kuroki, Nontawat Charoenphakdee, Han Bao, Junya Honda, Issei Sato, and
  Masashi Sugiyama.
\newblock Unsupervised domain adaptation based on source-guided discrepancy.
\newblock In \emph{AAAI}, volume~33, pages 4122--4129, 2019{\natexlab{a}}.

\bibitem[Kuroki et~al.(2019{\natexlab{b}})Kuroki, Charoenphakdee, Bao, Honda,
  Sato, and Sugiyama]{kuroki2019unsupervised}
Seiichi Kuroki, Nontawat Charoenphakdee, Han Bao, Junya Honda, Issei Sato, and
  Masashi Sugiyama.
\newblock Unsupervised domain adaptation based on source-guided discrepancy.
\newblock In \emph{AAAI}, volume~33, pages 4122--4129, 2019{\natexlab{b}}.

\bibitem[Langford and Caruana(2001)]{langford2001not}
John Langford and Rich Caruana.
\newblock (not) bounding the true error.
\newblock \emph{NeurIPS}, 14, 2001.

\bibitem[Langford and Seeger(2001)]{apx_langford2001bounds}
John Langford and Matthias Seeger.
\newblock Bounds for averaging classifiers.
\newblock \emph{CMU Technical Report}, 2001.

\bibitem[LeCun and Cortes(2010)]{apx_lecun-mnisthandwrittendigit-2010}
Yann LeCun and Corinna Cortes.
\newblock {MNIST} handwritten digit database.
\newblock \emph{web}, 2010.
\newblock URL \url{http://yann.lecun.com/exdb/mnist/}.

\bibitem[Li et~al.(2017{\natexlab{a}})Li, Yang, Song, and
  Hospedales]{apx_li2017deeper_PACS}
Da~Li, Yongxin Yang, Yi-Zhe Song, and Timothy~M Hospedales.
\newblock Deeper, broader and artier domain generalization.
\newblock In \emph{IEEE ICCV}, pages 5542--5550, 2017{\natexlab{a}}.

\bibitem[Li et~al.(2017{\natexlab{b}})Li, Yang, Song, and
  Hospedales]{li2017deeper_PACS}
Da~Li, Yongxin Yang, Yi-Zhe Song, and Timothy~M Hospedales.
\newblock Deeper, broader and artier domain generalization.
\newblock In \emph{IEEE ICCV}, pages 5542--5550, 2017{\natexlab{b}}.

\bibitem[Li and Bilmes(2007{\natexlab{a}})]{apx_li2007bayesian}
Xiao Li and Jeff Bilmes.
\newblock A bayesian divergence prior for classiffier adaptation.
\newblock In \emph{AISTATS}, pages 275--282. PMLR, 2007{\natexlab{a}}.

\bibitem[Li and Bilmes(2007{\natexlab{b}})]{li2007bayesian}
Xiao Li and Jeff Bilmes.
\newblock A bayesian divergence prior for classiffier adaptation.
\newblock In \emph{AISTATS}, pages 275--282. PMLR, 2007{\natexlab{b}}.

\bibitem[Lipton et~al.(2018{\natexlab{a}})Lipton, Wang, and
  Smola]{apx_lipton2018detecting}
Zachary Lipton, Yu-Xiang Wang, and Alexander Smola.
\newblock Detecting and correcting for label shift with black box predictors.
\newblock In \emph{ICML}, pages 3122--3130. PMLR, 2018{\natexlab{a}}.

\bibitem[Lipton et~al.(2018{\natexlab{b}})Lipton, Wang, and
  Smola]{lipton2018detecting}
Zachary Lipton, Yu-Xiang Wang, and Alexander Smola.
\newblock Detecting and correcting for label shift with black box predictors.
\newblock In \emph{ICML}, pages 3122--3130. PMLR, 2018{\natexlab{b}}.

\bibitem[Long et~al.(2017)Long, Zhu, Wang, and Jordan]{long2017deep}
Mingsheng Long, Han Zhu, Jianmin Wang, and Michael~I Jordan.
\newblock Deep transfer learning with joint adaptation networks.
\newblock In \emph{ICML}, pages 2208--2217. PMLR, 2017.

\bibitem[Long et~al.(2018)Long, Cao, Wang, and Jordan]{long2018conditional}
Mingsheng Long, Zhangjie Cao, Jianmin Wang, and Michael~I Jordan.
\newblock Conditional adversarial domain adaptation.
\newblock In \emph{NeurIPS}, pages 1647--1657, 2018.

\bibitem[Magliacane et~al.(2018{\natexlab{a}})Magliacane, van Ommen, Claassen,
  Bongers, Versteeg, and Mooij]{apx_magliacane2018domain}
Sara Magliacane, Thijs van Ommen, Tom Claassen, Stephan Bongers, Philip
  Versteeg, and Joris~M Mooij.
\newblock Domain adaptation by using causal inference to predict invariant
  conditional distributions.
\newblock In \emph{NeurIPS}, 2018{\natexlab{a}}.

\bibitem[Magliacane et~al.(2018{\natexlab{b}})Magliacane, van Ommen, Claassen,
  Bongers, Versteeg, and Mooij]{magliacane2018domain}
Sara Magliacane, Thijs van Ommen, Tom Claassen, Stephan Bongers, Philip
  Versteeg, and Joris~M Mooij.
\newblock Domain adaptation by using causal inference to predict invariant
  conditional distributions.
\newblock In \emph{NeurIPS}, 2018{\natexlab{b}}.

\bibitem[Mann and Thompson(1987)]{apx_mann1987rhetorical}
William~C Mann and Sandra~A Thompson.
\newblock \emph{Rhetorical structure theory: A theory of text organization}.
\newblock University of Southern California, Information Sciences Institute Los
  Angeles, 1987.

\bibitem[Mansour et~al.(2009{\natexlab{a}})Mansour, Mohri, and
  Rostamizadeh]{apx_mansour2009domain}
Yishay Mansour, Mehryar Mohri, and Afshin Rostamizadeh.
\newblock Domain adaptation: Learning bounds and algorithms.
\newblock \emph{arXiv:0902.3430}, 2009{\natexlab{a}}.

\bibitem[Mansour et~al.(2009{\natexlab{b}})Mansour, Mohri, and
  Rostamizadeh]{mansour2009domain}
Yishay Mansour, Mehryar Mohri, and Afshin Rostamizadeh.
\newblock Domain adaptation: Learning bounds and algorithms.
\newblock \emph{arXiv:0902.3430}, 2009{\natexlab{b}}.

\bibitem[Marcus et~al.(1993)Marcus, Santorini, and
  Marcinkiewicz]{apx_marcus1993building}
Mitchell Marcus, Beatrice Santorini, and Mary~Ann Marcinkiewicz.
\newblock Building a large annotated corpus of english: The penn treebank.
\newblock \emph{Computational Linguistics}, 1993.

\bibitem[Maurer(2004{\natexlab{a}})]{apx_maurer2004note}
Andreas Maurer.
\newblock A note on the pac bayesian theorem.
\newblock \emph{arXiv cs/0411099}, 2004{\natexlab{a}}.

\bibitem[Maurer(2004{\natexlab{b}})]{maurer2004note}
Andreas Maurer.
\newblock A note on the pac bayesian theorem.
\newblock \emph{arXiv cs/0411099}, 2004{\natexlab{b}}.

\bibitem[McAllester(2013)]{apx_mcallester2013pac}
David McAllester.
\newblock A {PAC}-{B}ayesian tutorial with a dropout bound.
\newblock \emph{arXiv:1307.2118v1}, 2013.

\bibitem[McAllester(1999)]{apx_mcallester1999some}
David~A McAllester.
\newblock Some {PAC}-{B}ayesian theorems.
\newblock \emph{Machine Learning}, 37:\penalty0 355--363, 1999.

\bibitem[McNamara and Balcan(2017{\natexlab{a}})]{apx_mcnamara2017risk}
Daniel McNamara and Maria-Florina Balcan.
\newblock Risk bounds for transferring representations with and without
  fine-tuning.
\newblock In \emph{ICML}, pages 2373--2381. PMLR, 2017{\natexlab{a}}.

\bibitem[McNamara and Balcan(2017{\natexlab{b}})]{mcnamara2017risk}
Daniel McNamara and Maria-Florina Balcan.
\newblock Risk bounds for transferring representations with and without
  fine-tuning.
\newblock In \emph{ICML}, pages 2373--2381. PMLR, 2017{\natexlab{b}}.

\bibitem[Nagarajan and Kolter(2019)]{nagarajan2019uniform}
Vaishnavh Nagarajan and J~Zico Kolter.
\newblock Uniform convergence may be unable to explain generalization in deep
  learning.
\newblock \emph{NeurIPS}, 32, 2019.

\bibitem[Netzer et~al.(2011)Netzer, Wang, Coates, Bissacco, Wu, and
  Ng]{apx_netzer2011reading}
Yuval Netzer, Tao Wang, Adam Coates, Alessandro Bissacco, Bo~Wu, and Andrew~Y
  Ng.
\newblock Reading digits in natural images with unsupervised feature learning.
\newblock \emph{NeurIPS Workshop on Deep Learning and Unsupervised Feature
  Learning 2011}, 2011.

\bibitem[Neyshabur(2017)]{neyshabur2017implicit}
Behnam Neyshabur.
\newblock Implicit regularization in deep learning.
\newblock \emph{arXiv:1709.01953}, 2017.

\bibitem[Neyshabur et~al.(2014)Neyshabur, Tomioka, and
  Srebro]{neyshabur2014search}
Behnam Neyshabur, Ryota Tomioka, and Nathan Srebro.
\newblock In search of the real inductive bias: On the role of implicit
  regularization in deep learning.
\newblock \emph{arXiv:1412.6614}, 2014.

\bibitem[Neyshabur et~al.(2017)Neyshabur, Bhojanapalli, Mcallester, and
  Srebro]{neyshabur2017exploring}
Behnam Neyshabur, Srinadh Bhojanapalli, David Mcallester, and Nati Srebro.
\newblock Exploring generalization in deep learning.
\newblock \emph{NeurIPS}, 30:\penalty0 5947--5956, 2017.

\bibitem[Parrado-Hern{\'a}ndez et~al.(2012)Parrado-Hern{\'a}ndez, Ambroladze,
  Shawe-Taylor, and Sun]{apx_parrado2012pac}
Emilio Parrado-Hern{\'a}ndez, Amiran Ambroladze, John Shawe-Taylor, and
  Shiliang Sun.
\newblock {PAC}-{B}ayes bounds with data dependent priors.
\newblock \emph{JMLR}, 13:\penalty0 3507--3531, 2012.

\bibitem[P{\'e}rez-Ortiz et~al.(2021{\natexlab{a}})P{\'e}rez-Ortiz, Rivasplata,
  Shawe-Taylor, and Szepesv{\'a}ri]{apx_perez2021tighter}
Mar{\i}a P{\'e}rez-Ortiz, Omar Rivasplata, John Shawe-Taylor, and Csaba
  Szepesv{\'a}ri.
\newblock Tighter risk certificates for neural networks.
\newblock \emph{JMLR}, 22, 2021{\natexlab{a}}.

\bibitem[P{\'e}rez-Ortiz et~al.(2021{\natexlab{b}})P{\'e}rez-Ortiz, Rivasplata,
  Shawe-Taylor, and Szepesv{\'a}ri]{perez2021tighter}
Mar{\i}a P{\'e}rez-Ortiz, Omar Rivasplata, John Shawe-Taylor, and Csaba
  Szepesv{\'a}ri.
\newblock Tighter risk certificates for neural networks.
\newblock \emph{JMLR}, 22, 2021{\natexlab{b}}.

\bibitem[Prasad et~al.(2008{\natexlab{a}})Prasad, Dinesh, Lee, Miltsakaki,
  Robaldo, Joshi, and Webber]{apx_prasad2008penn}
Rashmi Prasad, Nikhil Dinesh, Alan Lee, Eleni Miltsakaki, Livio Robaldo,
  Aravind~K Joshi, and Bonnie~L Webber.
\newblock The penn discourse treebank 2.0.
\newblock In \emph{LREC}. Citeseer, 2008{\natexlab{a}}.

\bibitem[Prasad et~al.(2008{\natexlab{b}})Prasad, Dinesh, Lee, Miltsakaki,
  Robaldo, Joshi, and Webber]{prasad2008penn}
Rashmi Prasad, Nikhil Dinesh, Alan Lee, Eleni Miltsakaki, Livio Robaldo,
  Aravind~K Joshi, and Bonnie~L Webber.
\newblock The penn discourse treebank 2.0.
\newblock In \emph{LREC}. Citeseer, 2008{\natexlab{b}}.

\bibitem[Rabanser et~al.(2019)Rabanser, G{\"u}nnemann, and
  Lipton]{rabanser2018failing}
Stephan Rabanser, Stephan G{\"u}nnemann, and Zachary~C Lipton.
\newblock Failing loudly: an empirical study of methods for detecting dataset
  shift.
\newblock In \emph{NeurIPS}, pages 1396--1408, 2019.

\bibitem[Ramesh and Yu(2010{\natexlab{a}})]{apx_ramesh2010identifying}
Balaji~Polepalli Ramesh and Hong Yu.
\newblock Identifying discourse connectives in biomedical text.
\newblock In \emph{AMIA Annual Symposium Proceedings}, volume 2010, page 657.
  American Medical Informatics Association, 2010{\natexlab{a}}.

\bibitem[Ramesh and Yu(2010{\natexlab{b}})]{ramesh2010identifying}
Balaji~Polepalli Ramesh and Hong Yu.
\newblock Identifying discourse connectives in biomedical text.
\newblock In \emph{AMIA Annual Symposium Proceedings}, volume 2010, page 657.
  American Medical Informatics Association, 2010{\natexlab{b}}.

\bibitem[Redko et~al.(2017{\natexlab{a}})Redko, Habrard, and
  Sebban]{apx_redko2017theoretical}
Ievgen Redko, Amaury Habrard, and Marc Sebban.
\newblock Theoretical analysis of domain adaptation with optimal transport.
\newblock In \emph{ECML PKDD}, pages 737--753. Springer, 2017{\natexlab{a}}.

\bibitem[Redko et~al.(2017{\natexlab{b}})Redko, Habrard, and
  Sebban]{redko2017theoretical}
Ievgen Redko, Amaury Habrard, and Marc Sebban.
\newblock Theoretical analysis of domain adaptation with optimal transport.
\newblock In \emph{ECML PKDD}, pages 737--753. Springer, 2017{\natexlab{b}}.

\bibitem[Redko et~al.(2020{\natexlab{a}})Redko, Morvant, Habrard, Sebban, and
  Bennani]{apx_redko2020ASO}
Ievgen Redko, Emilie Morvant, Amaury Habrard, Marc Sebban, and Youn{\`e}s
  Bennani.
\newblock A survey on domain adaptation theory.
\newblock \emph{ArXiv}, abs/2004.11829, 2020{\natexlab{a}}.

\bibitem[Redko et~al.(2020{\natexlab{b}})Redko, Morvant, Habrard, Sebban, and
  Bennani]{redko2020ASO}
Ievgen Redko, Emilie Morvant, Amaury Habrard, Marc Sebban, and Youn{\`e}s
  Bennani.
\newblock A survey on domain adaptation theory.
\newblock \emph{ArXiv}, abs/2004.11829, 2020{\natexlab{b}}.

\bibitem[Reimers and Gurevych(2019)]{apx_reimers2019sentence}
Nils Reimers and Iryna Gurevych.
\newblock Sentence-bert: Sentence embeddings using siamese bert-networks.
\newblock \emph{arXiv:1908.10084}, 2019.

\bibitem[Shalev-Shwartz and
  Ben-David(2014{\natexlab{a}})]{apx_shalev2014understanding}
Shai Shalev-Shwartz and Shai Ben-David.
\newblock \emph{Understanding machine learning: From theory to algorithms}.
\newblock Cambridge university press, 2014{\natexlab{a}}.

\bibitem[Shalev-Shwartz and
  Ben-David(2014{\natexlab{b}})]{shalev2014understanding}
Shai Shalev-Shwartz and Shai Ben-David.
\newblock \emph{Understanding machine learning: From theory to algorithms}.
\newblock Cambridge university press, 2014{\natexlab{b}}.

\bibitem[Shawe-Taylor and Williamson(1997)]{apx_shawe1997pac}
John Shawe-Taylor and Robert~C Williamson.
\newblock A {PAC} analysis of a {B}ayesian estimator.
\newblock In \emph{COLT}, 1997.

\bibitem[Shen et~al.(2018{\natexlab{a}})Shen, Qu, Zhang, and
  Yu]{apx_shen2018wasserstein}
Jian Shen, Yanru Qu, Weinan Zhang, and Yong Yu.
\newblock Wasserstein distance guided representation learning for domain
  adaptation.
\newblock In \emph{AAAI}, 2018{\natexlab{a}}.

\bibitem[Shen et~al.(2018{\natexlab{b}})Shen, Qu, Zhang, and
  Yu]{shen2018wasserstein}
Jian Shen, Yanru Qu, Weinan Zhang, and Yong Yu.
\newblock Wasserstein distance guided representation learning for domain
  adaptation.
\newblock In \emph{AAAI}, 2018{\natexlab{b}}.

\bibitem[Sugiyama et~al.(2007{\natexlab{a}})Sugiyama, Krauledat, and
  M{\"u}ller]{apx_sugiyama2007covariate}
Masashi Sugiyama, Matthias Krauledat, and Klaus-Robert M{\"u}ller.
\newblock Covariate shift adaptation by importance weighted cross validation.
\newblock \emph{JMLR}, 8\penalty0 (5), 2007{\natexlab{a}}.

\bibitem[Sugiyama et~al.(2007{\natexlab{b}})Sugiyama, Krauledat, and
  M{\"u}ller]{sugiyama2007covariate}
Masashi Sugiyama, Matthias Krauledat, and Klaus-Robert M{\"u}ller.
\newblock Covariate shift adaptation by importance weighted cross validation.
\newblock \emph{JMLR}, 8\penalty0 (5), 2007{\natexlab{b}}.

\bibitem[Tachet~des Combes et~al.(2020{\natexlab{a}})Tachet~des Combes, Zhao,
  Wang, and Gordon]{apx_tachet2020domain}
Remi Tachet~des Combes, Han Zhao, Yu-Xiang Wang, and Geoffrey~J Gordon.
\newblock Domain adaptation with conditional distribution matching and
  generalized label shift.
\newblock \emph{NeurIPS}, 33, 2020{\natexlab{a}}.

\bibitem[Tachet~des Combes et~al.(2020{\natexlab{b}})Tachet~des Combes, Zhao,
  Wang, and Gordon]{tachet2020domain}
Remi Tachet~des Combes, Han Zhao, Yu-Xiang Wang, and Geoffrey~J Gordon.
\newblock Domain adaptation with conditional distribution matching and
  generalized label shift.
\newblock \emph{NeurIPS}, 33, 2020{\natexlab{b}}.

\bibitem[Venkateswara et~al.(2017{\natexlab{a}})Venkateswara, Eusebio,
  Chakraborty, and Panchanathan]{apx_venkateswara2017deep}
Hemanth Venkateswara, Jose Eusebio, Shayok Chakraborty, and Sethuraman
  Panchanathan.
\newblock Deep hashing network for unsupervised domain adaptation.
\newblock In \emph{IEEE CVPR}, pages 5018--5027, 2017{\natexlab{a}}.

\bibitem[Venkateswara et~al.(2017{\natexlab{b}})Venkateswara, Eusebio,
  Chakraborty, and Panchanathan]{venkateswara2017deep}
Hemanth Venkateswara, Jose Eusebio, Shayok Chakraborty, and Sethuraman
  Panchanathan.
\newblock Deep hashing network for unsupervised domain adaptation.
\newblock In \emph{IEEE CVPR}, pages 5018--5027, 2017{\natexlab{b}}.

\bibitem[Wu et~al.(2019)Wu, Winston, Kaushik, and Lipton]{wu2019domain}
Yifan Wu, Ezra Winston, Divyansh Kaushik, and Zachary Lipton.
\newblock Domain adaptation with asymmetrically-relaxed distribution alignment.
\newblock In \emph{ICML}, pages 6872--6881. PMLR, 2019.

\bibitem[You et~al.(2019{\natexlab{a}})You, Wang, Long, and
  Jordan]{apx_you2019towards}
Kaichao You, Ximei Wang, Mingsheng Long, and Michael Jordan.
\newblock Towards accurate model selection in deep unsupervised domain
  adaptation.
\newblock In \emph{ICML}, pages 7124--7133. PMLR, 2019{\natexlab{a}}.

\bibitem[You et~al.(2019{\natexlab{b}})You, Wang, Long, and
  Jordan]{you2019towards}
Kaichao You, Ximei Wang, Mingsheng Long, and Michael Jordan.
\newblock Towards accurate model selection in deep unsupervised domain
  adaptation.
\newblock In \emph{ICML}, pages 7124--7133. PMLR, 2019{\natexlab{b}}.

\bibitem[Zeldes(2017)]{apx_zeldes2017gum}
Amir Zeldes.
\newblock The gum corpus: Creating multilayer resources in the classroom.
\newblock \emph{LREC}, 51\penalty0 (3):\penalty0 581--612, 2017.

\bibitem[Zeyrek et~al.(2020{\natexlab{a}})Zeyrek, Mendes, Grishina, Kurfal{\i},
  Gibbon, and Ogrodniczuk]{apx_zeyrek2020ted}
Deniz Zeyrek, Am{\'a}lia Mendes, Yulia Grishina, Murathan Kurfal{\i}, Samuel
  Gibbon, and Maciej Ogrodniczuk.
\newblock Ted multilingual discourse bank (ted-mdb): a parallel corpus
  annotated in the pdtb style.
\newblock \emph{LREC}, 54\penalty0 (2):\penalty0 587--613, 2020{\natexlab{a}}.

\bibitem[Zeyrek et~al.(2020{\natexlab{b}})Zeyrek, Mendes, Grishina, Kurfal{\i},
  Gibbon, and Ogrodniczuk]{zeyrek2020ted}
Deniz Zeyrek, Am{\'a}lia Mendes, Yulia Grishina, Murathan Kurfal{\i}, Samuel
  Gibbon, and Maciej Ogrodniczuk.
\newblock Ted multilingual discourse bank (ted-mdb): a parallel corpus
  annotated in the pdtb style.
\newblock \emph{LREC}, 54\penalty0 (2):\penalty0 587--613, 2020{\natexlab{b}}.

\bibitem[Zhang et~al.(2017)Zhang, Bengio, Hardt, Recht, and
  Vinyals]{zhang2017understanding}
Chiyuan Zhang, Samy Bengio, Moritz Hardt, Benjamin Recht, and Oriol Vinyals.
\newblock Understanding deep learning requires rethinking generalization.
\newblock In \emph{{ICLR} 2017}. OpenReview.net, 2017.

\bibitem[Zhang et~al.(2015{\natexlab{a}})Zhang, Gong, and
  Sch{\"o}lkopf]{apx_zhang2015multi}
Kun Zhang, Mingming Gong, and Bernhard Sch{\"o}lkopf.
\newblock Multi-source domain adaptation: A causal view.
\newblock In \emph{AAAI}, 2015{\natexlab{a}}.

\bibitem[Zhang et~al.(2015{\natexlab{b}})Zhang, Gong, and
  Sch{\"o}lkopf]{zhang2015multi}
Kun Zhang, Mingming Gong, and Bernhard Sch{\"o}lkopf.
\newblock Multi-source domain adaptation: A causal view.
\newblock In \emph{AAAI}, 2015{\natexlab{b}}.

\bibitem[Zhang et~al.(2019{\natexlab{a}})Zhang, Liu, Long, and
  Jordan]{apx_zhang2019bridging}
Yuchen Zhang, Tianle Liu, Mingsheng Long, and Michael Jordan.
\newblock Bridging theory and algorithm for domain adaptation.
\newblock In \emph{ICML}, pages 7404--7413. PMLR, 2019{\natexlab{a}}.

\bibitem[Zhang et~al.(2019{\natexlab{b}})Zhang, Liu, Long, and
  Jordan]{zhang2019bridging}
Yuchen Zhang, Tianle Liu, Mingsheng Long, and Michael Jordan.
\newblock Bridging theory and algorithm for domain adaptation.
\newblock In \emph{ICML}, pages 7404--7413. PMLR, 2019{\natexlab{b}}.

\bibitem[Zhao et~al.(2019{\natexlab{a}})Zhao, Des~Combes, Zhang, and
  Gordon]{apx_zhao2019learning}
Han Zhao, Remi~Tachet Des~Combes, Kun Zhang, and Geoffrey Gordon.
\newblock On learning invariant representations for domain adaptation.
\newblock In \emph{ICML}, pages 7523--7532. PMLR, 2019{\natexlab{a}}.

\bibitem[Zhao et~al.(2019{\natexlab{b}})Zhao, Des~Combes, Zhang, and
  Gordon]{zhao2019learning}
Han Zhao, Remi~Tachet Des~Combes, Kun Zhang, and Geoffrey Gordon.
\newblock On learning invariant representations for domain adaptation.
\newblock In \emph{ICML}, pages 7523--7532. PMLR, 2019{\natexlab{b}}.

\bibitem[Zhou et~al.(2020)Zhou, Yang, Hospedales, and Xiang]{apx_zhou2020deep}
Kaiyang Zhou, Yongxin Yang, Timothy Hospedales, and Tao Xiang.
\newblock Deep domain-adversarial image generation for domain generalisation.
\newblock \emph{arXiv:2003.06054}, 2020.

\bibitem[Zhou et~al.(2018)Zhou, Veitch, Austern, Adams, and
  Orbanz]{zhou2018non}
Wenda Zhou, Victor Veitch, Morgane Austern, Ryan~P Adams, and Peter Orbanz.
\newblock Non-vacuous generalization bounds at the imagenet scale: a
  pac-bayesian compression approach.
\newblock In \emph{ICLR}, 2018.

\end{thebibliography}
\clearpage
\onecolumn
\appendix
\section{Proofs}
\label{sec:proofs}
\subsection{Theorem~\ref{thm:ben2010theory}}
\begin{proof}
This is Thm. 2 of \citet{apx_ben2010theory} with added bound on $\mathbf{R}_S(h) - \mathbf{R}_\mathbb{S}(h)$ by standard uniform convergence arguments; e.g., Ch. 28.1 of \citet{apx_shalev2014understanding}. Boole's Inequality is used to combine bounds. 
\end{proof}
\subsection{Theorem~\ref{thm:germain2020pac} (Theorem~7 of Germain et al. [2020])}
\label{sec:germain2020pac}
\begin{theorem}\label{thm:germain2020pac}
\citep{apx_germain2020pac} Let $\mathcal{Y}$ be binary, $\mathbb{P}$ any distribution over $\mathcal{H}$, and $\omega > 0$. For all $\delta > 0$, w.p. at least $1-\delta$, for all distributions $\mathbb{Q}$ over $\mathcal{H}$,
\begin{equation}\small
\begin{split}
    & \mathbf{R}_\mathbb{T}(\mathbb{Q}) \leq \omega' ( \mathbf{R}_S(\mathbb{Q}) + |\mathrm{d}_S(\mathbb{Q}) - \mathrm{d}_T(\mathbb{Q})| )
    + |\mathrm{e}_\mathbb{S}(\mathbb{Q}) - \mathrm{e}_\mathbb{T}(\mathbb{Q})| + 2\omega \tfrac{\mathrm{KL}(\mathbb{Q} \mid \mid \mathbb{P})  - \ln ( \delta / 3) }{m\omega'} + 2(\omega' - 1)
\end{split}
\end{equation}
where $\omega' = 2\omega / (1 - \exp(-2\omega))$ and for $H_i \sim (\mathbb{Q})_i$, $(X,Y) \sim \mathbb{S}$ we have
\begin{equation}\small
\begin{split}
    & \mathrm{e}_\mathbb{S}(\mathbb{Q}) \defn \mathbf{E}[(1-\mathbf{1}_{\{H_1(X)\}}\{Y\}) (1-\mathbf{1}_{\{H_2(X)\}}\{Y\})], \\
    & \mathrm{d}_{\mathbb{S}}(\mathbb{Q}) \defn \mathbf{E} [1 - \mathbf{1}_{\{H_1(X)\}}\{H_2(X)\}].
\end{split}
\end{equation}
\end{theorem}
In comparison to Thm.~\ref{thm:ben2010theory}, the absolute difference in disagreement $\mathrm{d}$ is most similar to the $\mathcal{H}\Delta\mathcal{H}$-divergence and the absolute difference in joint-error $\mathrm{e}$ is most similar to the adaptability $\lambda$ \citep{apx_germain2020pac}. For this reason, in our discussion in Section~\ref{sec:background}, we refer to the former as the ``divergence'' and the latter as the ``adaptability''.
\begin{proof}
As noted, this is a simplification of Thm.~7 of \citet{apx_germain2020pac}. We set $\omega = a$ in the original notation and use the fact that $\omega / (1 - \exp(-\omega))$ is increasing for $\omega > 0$. \end{proof}
\subsection{Theorem~\ref{thm:pb-bound}}
Before diving into the proof, we setup some helpful notation and Lemmas.
\subsubsection{Notation}
Frequently in our proofs, we use the \textit{error gap}, defined for any distributions $\mathbb{S}, \mathbb{T}$ and hypothesis $h$
\begin{equation}\label{eqn:error_gap}
    \Delta_h(\mathbb{S}, \mathbb{T}) \defn \lvert \mathbf{R}_\mathbb{S}(h) - \mathbf{R}_\mathbb{T}(h) \rvert.
\end{equation}
By the identification in Eq.~\eqref{eqn:sample_pmf}, we observe that $\Delta_h(S, T)$ is also well-defined for any random samples $S$ and $T$. Also, using the usual definition of the Gibbs risk, $\Delta_\mathbb{Q}(\mathbb{S}, \mathbb{T})$ is well-defined for any distribution $\mathbb{Q}$ over a hypothesis space $\mathcal{H}$. Occasionally, we also use two-subscripts on the error-gap $\Delta$. The intended meaning is intuitive:
\begin{equation}
    \Delta_{q, p}(\mathbb{S}, \mathbb{T}) \defn \lvert \mathbf{R}_\mathbb{S}(q) - \mathbf{R}_\mathbb{T}(p) \rvert.
\end{equation}
This notation will be especially useful in proofs since $\Delta_{q, p}(\mathbb{S}, \mathbb{T})$ obeys a triangle-inequality with respect to the subscripts and arguments. Further, any bound on $\Delta_\mathbb{Q}(S, \mathbb{T})$ trivially yields a PAC-Bayesian adaptation bound for the Gibbs predictor $\mathbb{Q}$ by definition of the absolute value.

As another short-hand in proofs, we frequently use the following more evocative expressions for the indicator function:
\begin{equation}
    1[a = b] \defn \mathbf{1}_{\{a\}}\{b\}; \qquad 1[a \neq b] \defn  1 - \mathbf{1}_{\{a\}}\{b\}.
\end{equation}
Now, we can proceed with the employed Lemmas.
\subsubsection{Lemmas}
In this section, we build to the proof of Theorem~\ref{thm:pb-bound}. These results consist of most of the ``real'' work in proving the result. They range in degree of novelty and we provide some exposition on this point here. Lemma~\ref{lem:multi-class-triangle-eq} is an adaptation of the triangle-inequality for 01-loss \citep{apx_crammer2007learning, apx_ben2007analysis} to the multiclass setting. Similarly, Lemma~\ref{lem:ben-david} is an adaptation of the main inequality of \citet{apx_ben2010theory} to the multiclass setting. The former requires some work to verify the logic, while our overall strategy for the latter is similar to the binary case. Next, Lemma~\ref{lem:simple-da-bendavid} uses the identification in Eq.~\eqref{eqn:sample_pmf} to apply Lemma~\ref{lem:ben-david} to the random samples $S$ and $T$. While it is a simple insight, it is extremely important, since it enables us to introduce the sample-dependent adaptability $\tilde{\lambda}$. The next result, Lemma~\ref{lem:maurer}, is well-known in PAC-Bayes. Meanwhile, the final result, Lemma~\ref{lem:stoch_ben_david}, is a new result which allows us to apply Lemma~\ref{lem:simple-da-bendavid} to Gibbs predictors. When broken down in this manner, as is our intention, the individual pieces that build to our bound may appear simple. Still, it is important to remember that PAC-Bayesian bounds have never previously been combined with multiclass variants of the results of \citet{apx_ben2007analysis, apx_ben2010theory}. After some trial and error, we've found our primary innovations -- the use of sample-independent adaptability, along with Lemma~\ref{lem:stoch_ben_david} -- are vital to introducing the desired non-uniform notion of sample complexity. In any case, we now proceed by stating and proving each of the discussed Lemmas.
\begin{lemma}
\label{lem:multi-class-triangle-eq}
For any $(h, h') \in \mathcal{H}^2$ and any $(x,y) \in \mathcal{X} \times \mathcal{Y}$,
\begin{equation}\label{eqn:h_dis_y}
    1[h(x) \neq y] \leq 1[h(x) \neq h'(x)] + 1[h'(x) \neq y]
\end{equation}
and 
\begin{equation}\label{eqn:h_dis_h}
    1[h(x) \neq h'(x)] \leq 1[h'(x) \neq y] + 1[y \neq h(x)].
\end{equation}
\end{lemma}
\begin{proof}
We begin with Eq.~\eqref{eqn:h_dis_y}. We use proof by exhaustion. If $h(x) = y$, then the LHS is 0 and the RHS will always be non-negative so the equation is true. If $h(x) \neq y$ and $h(x) \neq h'(x)$, then the equation evaluates to $1 \leq 1 + c$  for $c \geq 0$ which is true. If $h(x) \neq y$ and $h(x) = h'(x)$, then $h'(x) \neq y$, and $1 \leq 1$ which is true. This concludes the argument.

Next, we consider Eq.~\eqref{eqn:h_dis_h}. Again, we use proof by exhaustion. If $h(x) =  h'(x)$, the LHS is 0. If $h(x) \neq h'(x)$ and $h(x) = y$, we have $h'(x) \neq y$ and the equation evaluates to $1 \leq 1$ which is true. If $h(x) \neq h'(x)$ and $h(x) \neq y$, it evaluates to $1 \leq 1 + c$ for $c \geq 0$ which is true and concludes the argument.
\end{proof}
Note, one observation is that the function $\tilde{d}(y, y') \defn 1[y' \neq y]$ for any arguments $y, y' \in \mathcal{Y}$ is identical to a well-known function called the trivial metric or the discrete metric. As implied by the name, the tuple $(\mathcal{Y}, \tilde{d})$ forms a \textit{metric space}, and subsequently, Lemma~\ref{lem:multi-class-triangle-eq} above is a simple consequence of this fact. Nonetheless, we maintain the proof above to keep our discussion relatively self-contained.  
\begin{lemma}
\label{lem:ben-david}
For any distributions $\mathbb{D}_1$ and $\mathbb{D}_2$ over $\mathcal{X} \times \mathcal{Y}$, for any $h \in \mathcal{H}$
\begin{equation}
    \mathbf{R}_{\mathbb{D}_1}(h) \leq \mathbf{R}_{\mathbb{D}_2}(h) + \mathbf{d}_{\mathcal{C}_h}((\mathbb{D}_1)_X, (\mathbb{D}_2)_X) + \min_{\eta \in \mathcal{H}} \Big \{\mathbf{R}_{\mathbb{D}_1}(\eta) + \mathbf{R}_{\mathbb{D}_2}(\eta) \Big \}
\end{equation}
where $\mathcal{C}_h = \mathcal{H}\Delta\mathcal{H}$ or $\mathcal{C}_h = h\Delta\mathcal{H}$ and $(\mathbb{D}_i)_X$ is the $\mathcal{X}$-marginal of $\mathbb{D}_i$.\footnote{In a formal sense, $(\mathbb{D}_i)_X$ is the pushforward distribution $\mathbb{D}_i \circ \pi^{-1} $of the projection $\pi : \mathcal{X} \times \mathcal{Y} \to \mathcal{X}$ defined $\pi(x,y) = x$.}
\end{lemma}
\begin{proof}
Let $\mathbb{D}_1$, $\mathbb{D}_2$, and $h$ as assumed. 

Recall by Lemma~\ref{lem:multi-class-triangle-eq} Eq.~\eqref{eqn:h_dis_y}, for any $h'$ in $\mathcal{H}$ and any $(x,y) \in \mathcal{X} \times \mathcal{Y}$
\begin{equation}
    1[h(x) \neq y] \leq 1[h(x) \neq h'(x)] + 1[h'(x) \neq y].
\end{equation}
Then, by monotonicity and linearity of the expectation, for any choice of $h'$,
\begin{equation}
\begin{split}
\mathbf{R}_{\mathbb{D}_1}(h) & \leq \mathbf{E}[1[h(X_1) \neq h'(X_1)]] + \mathbf{R}_{\mathbb{D}_1}(h'); \qquad X_1 \sim (\mathbb{D}_1)_X \\
& \leq \mathbf{E}[1[h(X_2) \neq h'(X_2)]] + \mathbf{R}_{\mathbb{D}_1}(h') + \xi; \qquad X_2 \sim (\mathbb{D}_2)_X
\end{split}
\end{equation}
where
\begin{equation}
\begin{split}
    \xi & = \big \lvert \mathbf{E}[1[h(X_2) \neq h'(X_2)]] - \mathbf{E}[1[h(X_1) \neq h'(X_1)]] \big \rvert \\
    & \leq \mathbf{d}_{\mathcal{C}_h}((\mathbb{D}_1)_X, (\mathbb{D}_2)_X)) \qquad \text{(by definition of supremum, for either choice of } \mathcal{C}_h).
\end{split}
\end{equation}
Alternatively, by Lemma~\ref{lem:multi-class-triangle-eq} Eq.~\eqref{eqn:h_dis_h}, for any choice of $h',x, y$,
\begin{equation}
    1[h(x) \neq h'(x)] \leq 1[h'(x) \neq y] + 1[y \neq h(x)].
\end{equation}
Using monotonicty and linearity of the expectation as before, we have
\begin{equation}
   \mathbf{E}[1[h(X_2) \neq h'(X_2)]] \leq \mathbf{R}_{\mathbb{D}_2}(h') + \mathbf{R}_{\mathbb{D}_2}(h); \qquad X_2 \sim (\mathbb{D}_2)_X.
\end{equation}
As the above holds for any $h' \in \mathcal{H}$, select $h'$ to be minimizer of the quantity $\mathbf{R}_{\mathbb{D}_1}(h') + \mathbf{R}_{\mathbb{D}_2}(h')$. 

This yields the desired result.
\end{proof}
\begin{lemma}
\label{lem:simple-da-bendavid}
Almost surely, w.r.t samples $S$ and $T$,
\begin{equation}\small
    \forall h \in \mathcal{H} \ : \ \Delta_h(S, T) \leq \tilde{\lambda} + \mathbf{d}_{\mathcal{C}_h}(S_X, T_X)
\end{equation}
where $\tilde{\lambda} \defn \min_{h \in \mathcal{H}} \mathbf{R}_S(h) + \mathbf{R}_T(h)$ and the bound holds for both $\mathcal{C}_h = \mathcal{H}\Delta\mathcal{H}$ and $\mathcal{C}_h = h\Delta\mathcal{H}$.
\end{lemma}
\begin{proof}
The statement asserts the following holds with probability 1 according to the random draws of $S$ and $T$:
\begin{equation}
    \forall h \in \mathcal{H} \ : \ \Delta_h(S, T) \leq \tilde{\lambda} + \mathbf{d}_{\mathcal{C}_h}(S_X, T_X)
\end{equation}
It is sufficient to show the statement holds for any realization of $S$ and $T$. Recall, for any realization, $S$ and $T$ themselves define distributions by the identification in Eq.~\eqref{eqn:sample_pmf}. So, Lemma~\ref{lem:ben-david} may be applied. Doing so twice and interchanging the roles of $S$ and $T$ gives
\begin{equation}
    \forall h \in \mathcal{H} \ : \ \mathbf{R}_h(S) - \mathbf{R}_h(T) \leq \tilde{\lambda} + \mathbf{d}_{\mathcal{C}_h}(S_X, T_X) \qquad \text{and} \qquad \mathbf{R}_h(T) - \mathbf{R}_h(S) \leq \tilde{\lambda} + \mathbf{d}_{\mathcal{C}_h}(S_X, T_X).
\end{equation}
So, the absolute difference between $\mathbf{R}_h(S)$ and $\mathbf{R}_h(T)$ is also bounded and we have our result.
\end{proof}
\begin{lemma} \citep{apx_maurer2004note}
\label{lem:maurer}
For any distribution $\mathbb{P}$ over $\mathcal{H}$, for any $\delta > 0$,
\begin{equation}
\mathbf{Pr} \Big ( \forall \ \mathbb{Q} \ : \ \Delta_{\mathbb{Q}}(T, \mathbb{T}) \leq \sqrt{\tfrac{\mathrm{KL}(\mathbb{Q} \mid \mid \mathbb{P}) + \ln \sqrt{4m} - \ln ( \delta) }{2m}}  \ \Big ) \geq 1 - \delta.
\end{equation}
\end{lemma}
\begin{proof}
This is the result of \citet{apx_maurer2004note} given below
\begin{equation}
    \mathbf{Pr} \Bigg ( \mathrm{kl}( \mathbf{R}_T(\mathbb{Q}) \mid \mid \mathbf{R}_\mathbb{T}(\mathbb{Q}) ) \leq \frac{\mathrm{KL}(\mathbb{Q} || \mathbb{P}) - \ln \delta + \ln \sqrt{4m} }{m} \ \Bigg ) \geq 1 - \delta,
\end{equation}
where the ``little'' $\mathrm{kl}$ is the KL-divergence between Bernoulli distributions parameterized by its arguments. The above bound implies the stated result by application of Pinsker's Inequality.
\end{proof}
\begin{lemma}
\label{lem:stoch_ben_david}
For any distribution $\mathbb{Q}$, almost surely w.r.t samples $S$ and $T$,
\begin{equation}
   \Delta_\mathbb{Q}(S, T) \leq \tilde{\lambda} + \mathbf{E}_H[\mathbf{d}_{\mathcal{C}_H}(S_X, T_X)]
\end{equation}
where $\tilde{\lambda}$ and $\mathcal{C}_H$ are defined as in Lemma~\ref{lem:simple-da-bendavid}.
\end{lemma}
\begin{proof}
We apply Lemma~\ref{lem:simple-da-bendavid}. By Jensen's Inequality, monotonicity of $\mathbf{E}$, and linearity of $\mathbf{E}$, we have
\begin{equation}
   \Delta_\mathbb{Q}(S, T) \leq \mathbf{E}_H[\Delta_H(S, T)] \leq \tilde{\lambda} + \mathbf{E}_H[\mathbf{d}_{\mathcal{C}_H}(S_X, T_X)]
\end{equation}
almost surely. In more details, for any realization of $S$ and $T$,
\begin{align*}
\Delta_\mathbb{Q}(S, T) & =  \big \lvert \mathbf{R}_\mathbb{Q}(S) - \mathbf{R}_\mathbb{Q}(T) \big \rvert & \\
& = \big \lvert \mathbf{E}[\mathbf{R}_S(H)] - \mathbf{E}[\mathbf{R}_T(H)] \big \rvert & (H \sim \mathbb{Q}, \ S \ \text{fixed}, \ T \ \text{fixed}) \\
& = \Big \lvert \mathbf{E} \Big [\mathbf{R}_S(H) - \mathbf{R}_T(H) \Big] \Big \rvert & \text{(Linearity of }\mathbf{E}) \\
& \leq \mathbf{E} \big [\Delta_H(S,T) \big ] & \text{(Jensen's Inequality)}\\
& \leq \mathbf{E} \big [ \tilde{\lambda} + \mathbf{d}_{\mathcal{C}_H}(S_X, T_X) \big ] & \text{(Lemma~\ref{lem:simple-da-bendavid} and monotonicity of }\mathbf{E}) \\
& \leq \tilde{\lambda} + \mathbf{E}[\mathbf{d}_{\mathcal{C}_H}(S_X, T_X)] & \text{(Linearity of }\mathbf{E}).
\end{align*}
\end{proof}
\subsubsection{Proof}
We give the final proof of Theorem~\ref{thm:pb-bound} below. Admittedly, it is a bit underwhelming, since most of the work has gone into the Lemmas above. The remaining component we rely on is our notation for the error-gap $\Delta$. By design, this notation exhibits a triangle-inequality.
\begin{proof}
Observe,
\begin{equation}\label{eqn:pattern-use-2}
    \Delta_{\mathbb{Q}}(S, \mathbb{T}) \leq \Delta_{\mathbb{Q}}(S, T) + \Delta_{\mathbb{Q}}(T, \mathbb{T}).
\end{equation}
To bound the former, we use Lemma~\ref{lem:stoch_ben_david}. To bound the latter, we use Lemma~\ref{lem:maurer}. We use Boole's Inequality to combine to the desired result. 
\end{proof}
\subsection{Theorem~\ref{thm:mid_div_red2erm}}
As noted in the main text, we employ the overall strategy of \citet{apx_ben2010theory}. The main distinction in our result below is the removal of any symmetry assumption on $\mathcal{H}$.
\begin{proof}
As before, we show the statement holds for any realization of $S_X$ and $T_X$. 

Let $\mathcal{C} = \mathcal{H}\Delta\mathcal{H}$ and expand the divergence as below
\begin{equation}
\begin{split}
    & \mathbf{d}_\mathcal{C}(S_X, T_X) = \max_{\varphi \in \mathcal{H} \Delta \mathcal{H}} \big \lvert \mathbf{E}[\varphi(X)] - \mathbf{E}[\varphi(\tilde{X})]\big \rvert = \max_{\varphi \in \mathcal{H} \Delta \mathcal{H}} \big \lvert \mathbf{Pr}(\varphi(X) = 1) - \mathbf{Pr}(\varphi(\tilde{X}) = 1)\big \rvert
\end{split}
\end{equation}
where $X \sim S_X$, $\tilde{X} \sim T_X$. Note, we substitute $\max$ for $\sup$ because both $S_X$ and $T_X$ are finitely supported, and thus, some $\varphi \in \mathcal{C}$ does achieve the maximum. Then, we have
\begin{equation}
\label{eqn:dis_divergence_optim}
\begin{split}
& \max_{\varphi \in \mathcal{H} \Delta \mathcal{H}} \big \lvert \mathbf{Pr}(\varphi(X) = 1) - \mathbf{Pr}(\varphi(\tilde{X}) = 1)\big \rvert \\
& = \max_{\varphi \in \mathcal{H} \Delta \mathcal{H}} \max \begin{rcases}
    \begin{dcases}
       \mathbf{Pr}(\varphi(X) = 1) - \mathbf{Pr}(\varphi(\tilde{X}) = 1), \\
      \mathbf{Pr}(\varphi(\tilde{X}) = 1) - \mathbf{Pr}(\varphi(X)=1)
    \end{dcases}
  \end{rcases} \\
& = \max_{\varphi \in \mathcal{H} \Delta \mathcal{H}} \max \begin{rcases}
    \begin{dcases}
       1 - \mathbf{Pr}(\varphi(X) =  0) - \mathbf{Pr}(\varphi(\tilde{X}) = 1), \\
      1 - \mathbf{Pr}(\varphi(\tilde{X}) =  0) - \mathbf{Pr}(\varphi(X) = 1)
    \end{dcases}
  \end{rcases} \\
& = \max \begin{rcases}
    \begin{dcases}
       1 - \min_ {\varphi \in \mathcal{H} \Delta \mathcal{H}} \Big \{ \mathbf{Pr}(\varphi(X) =  0) +  \mathbf{Pr}(\varphi(\tilde{X}) = 1) \Big \}, \\
      1 - \min_ {\varphi \in \mathcal{H} \Delta \mathcal{H}} \Big \{ \mathbf{Pr}(\varphi(\tilde{X}) =  0) + \mathbf{Pr}(\varphi(X) = 1) \Big \}
    \end{dcases}
  \end{rcases}. \\
\end{split}
\end{equation}
The first equality follows by definition of absolute value, the second by law of complements, and last because consecutive applications of the $\max$ operation may be interchanged. Taking $P,Q,U,V$ as assumed, the result follows by the definition of risk; i.e., Eq.~\eqref{eqn:risk}.
\end{proof}
\subsection{Theorem~\ref{thm:surrogate_loss}}
As we are aware, Theorem~\ref{thm:surrogate_loss} is the first proposal for approximation of ERM over the class $\mathcal{H}\Delta\mathcal{H}$ when $\mathcal{H}$ has multiclass output. Our strategy is to identify an appropriate score-based surrogate expression for any $\varphi \in \mathcal{S}\Delta\mathcal{S}$; i.e., which is positive where $\varphi$ returns 1 and negative otherwise. Upon doing so, we can use standard techniques for giving smooth upperbounds to the 01-loss.
\begin{proof}
Let $x \in \mathcal{X}$, $\mathbf{f}, \mathbf{g} \in \mathcal{F}$ and suppose $\mathbf{f}(x)$ and $\mathbf{g}(x)$ have no repeated entries. Recall, for any two sets of non-negative numbers $S_1$ and $S_2$ the following equality holds\footnote{Suppose not. Then, WLOG $\max \{a \cdot b\} = d \cdot e > (\max S_1) \cdot (\max S_1)$ for some $d \neq \max S_1$ or some $e \neq \max S_2$. But, we also have $d \cdot e \leq d \cdot \max S_2 \leq (\max S_1) \cdot (\max S_2)$, a contradiction.}
\begin{equation}
    \max \{a \cdot b \mid a \in S_1, b \in S_2\} = (\max S_1) \cdot (\max S_2).
\end{equation}
From this and the fact that $\tau$ is non-negative and order-preserving, we know $\mathbf{A}_{ii} \geq \mathbf{A}_{jk}$ for some $i \in [C]$ and all $(j,k) \in [C]^2$ if and only if 
\begin{equation}
i = \argmax_{\ell \in C} \mathbf{f}_\ell(x) = \argmax_{\ell \in C} \mathbf{g}_\ell(x).   
\end{equation}
Notice, ties are impossible due to the assumed uniqueness of the scores. So, by this same logic, we observe
\begin{equation}
\begin{split}
    & \argmax_{\ell \in C} \mathbf{f}_\ell(x) \neq \argmax_{\ell \in C} \mathbf{g}_\ell(x) \\
    \mathrm{iff} \quad & \forall \ i \in [C], \ \exists \ (j, k) \in [C]^2 \ : \mathbf{A}_{ii} < \mathbf{A}_{jk} \\
    \mathrm{iff} \quad & \max_{i \in [C]} \mathbf{A}_{ii} < \max_{(j,k) \in [C]^2} \mathbf{A}_{jk} \\
    \mathrm{iff} \quad & 0 < \max_{(j,k) \in [C]^2} \mathbf{A}_{jk} - \max_{i \in [C]} \mathbf{A}_{ii} = z(x)
\end{split}
\end{equation}
So, under the current assumptions, the score $z(x)$ is positive if and only if $\hat{y} = 1 - \mathbf{1}_{\{\Psi_\mathbf{f}(x)\}}\{\Psi_\mathbf{g}(x)\} = 1$. Using this fact, it is easy to verify $\mathcal{L}(z(x),y) \geq 1[\hat{y} \neq y]$ for each case $(\hat{y}, y) \in \{(0,0), (0,1), (1,0), (1,1)\}$. The loss $\mathcal{L}$ is actually a standard surrogate -- i.e., the cross-entropy -- multiplied by a constant factor as in \citet{apx_dziugaite2017computing} to turn it into a propper upperbound on the 01-loss. The main novelty here comes from defining $z(x)$ to be positive whenever $\hat{y}$ is. 

Notice, the inequality holds on all but a set of measure 0, according to $\mathbb{D}$. Thus, monotonicity of $\mathbf{E}$ gives the result.
\end{proof}
\subsection{Theorem~\ref{thm:mdp_div_red2erm}}
\label{sec:mdp_div_red2erm}
As noted in the main text, Theorem~\ref{thm:mdp_div_red2erm} is conceptually similar to a result -- in the binary case -- given by \citet{apx_kuroki2019unsupervised}. Unfortunately, their strategy does not simply extend to the multiclass case: there is a loss of precision due to the increased degrees of freedom in multiclass classification. As a result, we observe the need to add additional constraints on the labeling function for the classification problem. Specifically, we introduce the class $\Upsilon$ for use in our reduction. Careful attention is paid to show the constrained labeling function can be independent of the classifier we wish to learn $\varphi \in \mathcal{H}$, which enables our appeal to a simple heuristic that is also independent of $\varphi$. Otherwise, in simpler formulations, this dependence produces a more complicated minimization problem.
\begin{proof}
We show the statement holds for any realization of $S_X$ and $T_X$. 
 
Let $h \in \mathcal{H}$ arbitrarily and let $\mathcal{C} = h\Delta\mathcal{H}$. We proceed by expanding the divergence:
\begin{equation}\label{eqn:mdp_expansion}
\begin{split}
& \mathbf{d}_{\mathcal{C}}(S_X, T_X) = \max_{\nu \in h \Delta \mathcal{H}} \big \lvert \mathbf{Pr}(\nu(X) = 1) - \mathbf{Pr}(\nu(\tilde{X}) = 1)\big \rvert \\
& = \max \begin{rcases}
    \begin{dcases}
       1 - \min_ {\nu \in h \Delta \mathcal{H}} \Big \{ \mathbf{Pr}(\nu(X) =  0) +  \mathbf{Pr}(\nu(\tilde{X}) = 1) \Big \}, \\
      1 - \min_ {\nu \in h \Delta \mathcal{H}} \Big \{ \mathbf{Pr}(\nu(\tilde{X}) =  0) + \mathbf{Pr}(\nu(X) = 1) \Big \}
    \end{dcases}
  \end{rcases} \\
  & = \max \begin{rcases}
    \begin{dcases}
       1 - \min_ {\varphi \in \mathcal{H}} \Big \{ \mathbf{Pr}(h(X) = \varphi(X)) +  \mathbf{Pr}(h(\tilde{X}) \neq \varphi(\tilde{X})) \Big \}, \\
      1 - \min_ {\varphi \in \mathcal{H}} \Big \{ \mathbf{Pr}(h(\tilde{X}) =  \varphi(\tilde{X})) + \mathbf{Pr}(h(X) \neq \varphi(X)) \Big \}
    \end{dcases}
  \end{rcases}
\end{split}
\end{equation}
where $X \sim S_X$ and $\tilde{X} \sim T_X$. The first and second lines follow from an identical expansion as in the proof of Theorem~\ref{thm:mid_div_red2erm}. The last follows by definition of $h \Delta \mathcal{H}$.

Next, we observe
\begin{equation}
\label{eqn:impercision_introduced}
    \forall \varphi \in \mathcal{H}, \ \forall \bar{h} \in \Upsilon \ : \ \mathbf{Pr}(h(X) = \varphi(X)) +  \mathbf{Pr}(h(\tilde{X}) \neq \varphi(\tilde{X})) \leq \mathbf{Pr}(\bar{h}(X) \neq \varphi(X)) +  \mathbf{Pr}(h(\tilde{X}) \neq \varphi(\tilde{X})).
\end{equation}
The inequality follows by the monotonicity of probability and the fact 
\begin{equation}
\{x \mid h(x) = \varphi(x)\} \subseteq \{x \mid \bar{h}(x) \neq \varphi(x)\} \qquad\text{(by definition of }\bar{h}).
\end{equation}
Meanwhile, setting 
\begin{equation}
    \bar{h}^*_\varphi(x) \defn \begin{cases}
    \varphi(x), & \text{if} \ \varphi(x) \neq h(x) \\
    \max \{\ell \in [C] \mid \ell \neq \varphi(x)\}, & \text{else}
    \end{cases}
\end{equation}
implies $\bar{h}^*_\varphi \in \Upsilon$ and $\{x \mid h(x) = \varphi(x)\} = \{x \mid \bar{h}^*_\varphi(x) \neq \varphi(x)\}$. So, we also have
\begin{equation}
\label{eqn:impercision_resolved}
    \forall \varphi \in \mathcal{H} \ : \ \mathbf{Pr}(h(X) = \varphi(X)) +  \mathbf{Pr}(h(\tilde{X}) \neq \varphi(\tilde{X}))  = \mathbf{Pr}(\bar{h}^*_\varphi(X) \neq \varphi(X)) +  \mathbf{Pr}(h(\tilde{X}) \neq \varphi(\tilde{X})).
\end{equation}
Considering that $\bar{h}^*_\varphi \in \Upsilon$, Eq.~\eqref{eqn:impercision_introduced} and Eq.~\eqref{eqn:impercision_resolved} in combination tell us 
\begin{equation}
\label{eqn:mdp_erm_equality-1}
    \forall \varphi \in \mathcal{H} \ : \ \mathbf{Pr}(h(X) = \varphi(X)) +  \mathbf{Pr}(h(\tilde{X}) \neq \varphi(\tilde{X}))  = \min\nolimits_{\bar{h} \in \Upsilon} \Big \{ \mathbf{Pr}(\bar{h}(X) \neq \varphi(X)) +  \mathbf{Pr}(h(\tilde{X}) \neq \varphi(\tilde{X})) \Big \}.
\end{equation}
To see this, it's easiest to use the definition of a set's $\min$ element as that which attains the greatest lower bound; i.e., the $\inf$ or infimum. Then, Eq.~\eqref{eqn:impercision_introduced} implies the $\min$ upperbounds the LHS of Eq.~\eqref{eqn:mdp_erm_equality-1}, and Eq.~\eqref{eqn:impercision_resolved} implies the $\min$ lowerbounds the LHS of Eq.~\eqref{eqn:mdp_erm_equality-1}. In combination, these bounds prove equality.

Note, an identical argument also gives,
\begin{equation}
\label{eqn:mdp_erm_equality-2}
    \forall \varphi \in \mathcal{H} \ : \ \mathbf{Pr}(h(\tilde{X}) =  \varphi(\tilde{X})) + \mathbf{Pr}(h(X) \neq \varphi(X)) = \min\nolimits_{\bar{h} \in \Upsilon} \Big \{ \mathbf{Pr}(\bar{h}(\tilde{X}) \neq \varphi(\tilde{X})) + \mathbf{Pr}(h(X) \neq \varphi(X)) \Big \}.
\end{equation}
To continue, we apply Eq.~\eqref{eqn:mdp_erm_equality-1} and Eq.~\eqref{eqn:mdp_erm_equality-2} to Eq.~\eqref{eqn:mdp_expansion}.
Specifically, we have
\begin{equation}
\begin{split}
    & \max \begin{rcases}
    \begin{dcases}
       1 - \min_ {\varphi \in \mathcal{H}} \Big \{ \mathbf{Pr}(h(X) = \varphi(X)) +  \mathbf{Pr}(h(\tilde{X}) \neq \varphi(\tilde{X})) \Big \}, \\
      1 - \min_ {\varphi \in \mathcal{H}} \Big \{ \mathbf{Pr}(h(\tilde{X}) =  \varphi(\tilde{X})) + \mathbf{Pr}(h(X) \neq \varphi(X)) \Big \}
    \end{dcases}
  \end{rcases} \\
  & = \max \begin{rcases}
    \begin{dcases}
       1 - \underset{\bar{h} \in \Upsilon}{\min_ {\varphi \in \mathcal{H},}} \Big \{ \mathbf{Pr}(\bar{h}(X) \neq \varphi(X)) +  \mathbf{Pr}(h(\tilde{X}) \neq \varphi(\tilde{X})) \Big \}, \\
      1 - \underset{\bar{h} \in \Upsilon}{\min_ {\varphi \in \mathcal{H},}} \Big \{ \mathbf{Pr}(\bar{h}(\tilde{X}) \neq \varphi(\tilde{X})) + \mathbf{Pr}(h(X) \neq \varphi(X)) \Big \}
    \end{dcases}
  \end{rcases}.
\end{split}
\end{equation}
Taking $P, Q, U, V$ as assumed, the desired result follows by the definition of risk in Eq.~\eqref{eqn:risk}.
\end{proof}
\subsection{Theorem~\ref{thm:pb-bound-efficient}}
As noted in the main text, this result introduces a deterministic reference to avoid costly Monte-Carlo estimation. It is the consequence of a series of triangle-inequalities and some of the Lemmas disucssed in proof of Thm.~\ref{thm:pb-bound}. 
\begin{proof}
Observe, for any $h_*$,
\begin{equation}
\begin{split}\label{eqn:pattern-use-3}
    \Delta_{\mathbb{Q}}(S, \mathbb{T}) & \leq \Delta_{\mathbb{Q}, h_*}(S, T) + \Delta_{h_*,\mathbb{Q}}(T, \mathbb{T}) \\
    & \leq \Delta_{h_*}(S, T) + \Delta_{\mathbb{Q}, h_*}(S, S) + \Delta_{\mathbb{Q}}(T, \mathbb{T}) + \Delta_{\mathbb{Q}, h_*}(T, T) \\
    & \leq \rho + \Delta_{h_*}(S, T) + \Delta_{\mathbb{Q}}(T, \mathbb{T})
\end{split}
\end{equation}
Use Lemma~\ref{lem:simple-da-bendavid} and Lemma~\ref{lem:maurer}, respectively, to bound the latter two terms. Application of Boole's Inequality and selection of $h_* = \mu$ gives the result.
\end{proof}
\subsection{Corollary~\ref{cor:pb-bound-efficient}}
Conceptually, this result relies on the same proof-technique as Theorem~\ref{thm:pb-bound-efficient}, but the proof is still a bit more technically involved than a typical ``Corollary'' because it requires the measure-theoretic notion of a pushforward. We consider pushfowards of empirical distributions, which are finitely supported, so there is no need to discuss issues of measurability.
\begin{proof}
Following the proof of Theorem~\ref{thm:pb-bound-efficient}, we have $\Delta_{\mathbb{Q}}(S, \mathbb{T}) \leq \rho + \Delta_{\mu}(S, T) + \Delta_{\mathbb{Q}}(T, \mathbb{T})$. Now, recalling $\mu$ is the composition of a classifier $c_\mu$ and a feature extractor $f_\mu$, we have
\begin{equation}\label{eqn:rewrite}
    \Delta_\mu(S, T) = \Delta_{c_\mu}(S \circ f_\mu^{-1}, T \circ f_\mu^{-1})
\end{equation}
where we abuse notation and write $\mathbb{D} \circ g^{-1}$ for the pushforward of a distribution $\mathbb{D}$ on $\mathcal{X}\times\mathcal{Y}$ by the function $\Phi_g(x,y) = (g(x), y)$. In details, setting $\ell(h, (x,y)) = 1[h(x) \neq y]$ and assuming $f_\mu : \mathcal{X} \to \mathcal{Z}$ and $c_\mu : \mathcal{Z} \to \mathcal{Y}$, Eq.~\eqref{eqn:rewrite} follows because
\begin{equation}\label{eqn:pf-expl}
    \mathbf{R}_\mathbb{D}(\mu) = \int_{\mathcal{X} \times \mathcal{Y}} \ell(c_\mu \circ f_\mu, v)\mathbb{D}(\mathrm{d}v) = \int \ell(c_\mu, \Phi_{f_\mu}(v))\mathbb{D}(\mathrm{d}v) = \int_{\mathcal{Z} \times\mathcal{Y}} \ell(c_\mu, w)\mathbb{D} \circ f_\mu^{-1}(\mathrm{d}w) = \mathbf{R}_{\mathbb{D} \circ f_\mu^{-1}}(c_\mu)
\end{equation}
for any distribution $\mathbb{D}$ over $\mathcal{X} \times \mathcal{Y}$. After applying the equality in Eq.~\eqref{eqn:rewrite}, we can conclude our argument as in the proof of Theorem~\ref{thm:pb-bound} using Lemma~\ref{lem:simple-da-bendavid}. Although, it should be noted the adaptation problem has changed slightly, since we now consider the hypothesis space $\mathcal{W} = \{c_h \mid h \in \mathcal{H}\} \subseteq \mathcal{Y}^\mathcal{Z}$, the source distribution $\mathbb{S} \circ f_\mu^{-1}$ over $\mathcal{Z} \times \mathcal{Y}$, and the target distribution $\mathbb{T} \circ f_\mu^{-1}$ over $\mathcal{Z} \times \mathcal{Y}$. Of course, Lemma~\ref{lem:simple-da-bendavid} still applies in this case, so this does not present an issue.  

After this, to arrive at the result in the main text, we simplify terms to remove any discussion of pushforward distributions. For any risks, this is accomplished by reversing the steps in Eq.~\eqref{eqn:pf-expl}. For any divergences, a similar equality holds and can be applied. In particular, for any $\mathcal{Q} \subseteq \mathcal{Y}^\mathcal{Z}$, any function $p : \mathcal{X} \to \mathcal{Z}$, and any distributions $\mathbb{S}$ and $\mathbb{T}$ over $\mathcal{X} \times \mathcal{Y}$, we use the expansion below:
\begin{equation}
\begin{split}
    \mathbf{d}_\mathcal{Q}((\mathbb{S} \circ p^{-1})_Z, (\mathbb{T} \circ p^{-1})_Z) & = \sup_{q \in \mathcal{Q}} \lvert \mathbf{E}_{Z \sim (\mathbb{S} \circ p^{-1})_Z} [q(Z)] - \mathbf{E}_{Z \sim (\mathbb{T} \circ p^{-1})_Z} [q(Z)]\rvert \\
    & = \sup_{q \in \mathcal{Q}} \lvert \mathbf{E}_{X \sim \mathbb{S}_X} [(q \circ p) (X)] - \mathbf{E}_{X \sim \mathbb{T}_X} [(q \circ p) (X)] \rvert \qquad (\text{similiar to Eq.~\eqref{eqn:pf-expl}})\\
    & = \sup_{r \in \mathcal{Q} \circ p^{-1}} \lvert \mathbf{E}_{X \sim \mathbb{S}_X} [r (X)] - \mathbf{E}_{X \sim \mathbb{T}_X} [r (X)] \rvert \qquad (\mathcal{Q} \circ p^{-1} \defn \{q \circ p \mid q \in \mathcal{Q}\})\\
    & = \mathbf{d}_{\mathcal{Q} \circ p^{-1}}(S_X, T_X).
\end{split}
\end{equation}
Taking $\mathcal{Q} = \{1 - \mathbf{1}_{\{c(\cdot)\}}\{c'(\cdot)\} \mid (c,c') \in \mathcal{W}\}$ and $p = f_\mu$, we end up with $\mathcal{Q} \circ p^{-1} = [\mathcal{H}\Delta\mathcal{H}]_\mu$ as defined in the main text. Likewise, taking $\mathcal{Q} = \{1 - \mathbf{1}_{\{c_\mu(\cdot)\}}\{c'(\cdot)\} \mid c' \in \mathcal{W}\}$ and $p = f_\mu$, we end up with $\mathcal{Q} \circ p^{-1} = [\mu\Delta\mathcal{H}]_\mu$.
\end{proof}

\section{Extended Related Works}
\label{sec:ext_related}
Here, we give an extended version of the related works (Section~\ref{sec:related}). First, we discuss theoretical adaptation work. We compartmentalize relevant contributions based on some key-terms common to adaptation bounds. Following this, we discuss related works in PAC-Bayes, in which, we give a more in depth history of these bounds.

\paragraph{Divergence} Many bounds use a modified, or generalized, divergence term. \citet{apx_mansour2009domain} define divergence for \textit{any} loss function (i.e., in addition to the 01-loss we consider). 
With some restrictions on hypothesis space, \citet{apx_redko2017theoretical} show a Wasserstein metric may be used to bound error. \citet{apx_shen2018wasserstein} extend this to more general settings. As noted by \citet{apx_redko2020ASO}, bounds based on Wasserstein metric imply bounds based on MMD \citep{apx_gretton2012kernel} due to a general relationship between the two. \citet{apx_johansson2019support} give another bound based on an integral probability metric. Note, none of these works consider approximation of divergences used to bound 01-loss in multiclass settings. In this regard, the closest work to ours is \citet{apx_zhang2019bridging} who approximate a divergence used to bound a multiclass \textit{margin} loss, which in turn, bounds the 01-loss we consider. As noted, the primary difference between our work and the work of \citet{apx_zhang2019bridging} is the use of uniform sample-complexity in the latter. Possibly, bounds in the latter could be extended to PAC-Bayesian contexts as well, but our choice of divergences allows us to work directly with 01-loss and avoid any loosening of the bound via the margin penalty.  

\paragraph{Adaptability} Besides requiring small adaptability term, some theoretical DA works consider other possible assumptions. For example, a covariate shift assumption can be made: the marginal feature distributions disagree, but the feature-conditional label distributions are identical. This assumption is useful, for example, in designing model-selection algorithms \citep{apx_sugiyama2007covariate, apx_you2019towards}, but \citet{apx_david2010impossibility} show this assumption (on its own) is \textit{not} enough for the general DA problem. Another frequent assumption is label-shift: the marginal label distributions disagree, but the label-conditional feature distributions remain the same. As mentioned, \citet{apx_zhao2019learning} show failure-cases in this context, while \citet{apx_lipton2018detecting} propose techniques for detecting and correcting shift in this case. Similarly, \citet{apx_tachet2020domain} propose \textit{generalized} label-shift and motivate new algorithms in this context. The DA problem can also be modeled through causal graphs \citep{apx_zhang2015multi, apx_magliacane2018domain} and some extensions to DA consider a meta-distribution over targets \citep{apx_blanchard2021domain, apx_albuquerque2019adversarial, apx_deng2020representation}. Notably, most assumptions are untestable in practice, but not many works consider this. As we are aware, we are the first work to use a sample-dependent adaptability term, which improves estimation in empirical study.


\paragraph{PAC-Bayes} 
For completeness, besides what is discussed here, readers are directed to the work of \citet{apx_catoni2007pac}, \citet{apx_mcallester2013pac}, \citet{apx_germain2009pac, apx_germain2015risk}, and the primer by \citet{apx_guedj2019primer}. While PAC-Bayes is often attributed to \citet{apx_mcallester1999some} with early ideas by \citet{apx_shawe1997pac}, the particular bound we use is due to \citet{apx_maurer2004note}. A similar result was first shown by \citet{apx_langford2001bounds} for 01-loss. In experiments, we use data-dependent priors, perhaps first conceptualized by \citet{apx_ambroladze2007tighter, apx_parrado2012pac}. Besides the previously discussed work of Germain et al., PAC-Bayes has also been used in theories for transfer learning \citep{apx_li2007bayesian, apx_mcnamara2017risk}. As mentioned, our bounds are the first PAC-Bayesian multiclass adaptation bounds.
\section{Experimental Details}
\label{sec:exp_details}
\subsection{Datasets and Models}
\label{sec:datasets_and_models}
As noted in the main text, we consider a collection of common adaptation datasets from both computer vision and NLP. Each dataset consists of a number of component \textit{domains} which are themselves distinct datasets that all share a common label space. In this way, we can simulate transfer of some model from one domain to another. The datasets and models we consider are as follows:
\begin{enumerate}
     \item \textbf{Digits}: Digits consists of collection of digit classification datasets including: USPS \citep{apx_uspsdataset}, MNIST \citep{apx_lecun-mnisthandwrittendigit-2010}, and SVHN \citep{apx_netzer2011reading}. We use only the training sets. The number images in each is about 7K, 60K, and 70K, respectively. We select $\mathcal{X}$ to be the space of 28$\times$28 grayscale images (i.e., the original feature space for MNIST). For USPS and SVHN, this is accomplished through image transformation. The label space $\mathcal{Y}$ consists of the digits 0-9. As we are aware, this collection was first used by \citet{apx_ganin2015unsupervised}. For this task, we consider $\mathcal{H}$ to be a space of CNNs of a fixed 4-layer architecture.
     \item \textbf{PACS}: PACS is an image-classification, domain generalization dataset where each domain has a different style. It was proposed by \citet{apx_li2017deeper_PACS} to be a more challenging task compared to existing generalization datasets. The domains consist of images in style of: Photo, Art Painting, Cartoon, or Sketch. There are about 10K total labeled images with some slight imbalance in the liklihood of each style. The label space $\mathcal{Y}$ consists of 7 common object categories: dog, elephant, giraffe, guitar, horse, house, and person. The feature space $\mathcal{X}$ is selected to be space of real-vectors of dimension $2048$; i.e., $\mathbb{R}^{2048}$. To map to the feature space from an image, we use the hidden-layer output of an image passed through a pre-trained ResNet-50 \citep{apx_he2016deep}. For this task, we consider $\mathcal{H}$ to be either the space of linear classifiers or the space of 4-layer FCNs of a fixed architecture (fully-connected networks).
     \item \textbf{Office-Home}: Office-Home was originally proposed by \citet{apx_venkateswara2017deep}, but we use the smaller preprocessed version given by \citet{apx_zhou2020deep}. The dataset is similar to PACS. It also contains 4 different styles as its component domains across about 15K total images: Art, Clipart, Product, and Real-World. Unlike PACS, it has a much larger number of classes. In particular, the label space $\mathcal{Y}$ contains 65 categories of different daily objects. Like PACS, we use the outputs of ResNet-50 to map to the real vector space $\mathcal{X}$. We let $\mathcal{H}$ be either a linear model or a 4-layer FCN as before.
     \item \textbf{Amazon Reviews}: Amazon Reviews is a text-classification dataset introduced by \citet{apx_blitzer2007biographies}. We use the Books, DVD, Electronics, and Kitchen domains preprocessed as in Blitzer et al. This totals about 4000 reviews which are labeled as having positive or negative sentiment. The feature space $\mathcal{X}$ is the space $\mathbb{N}^{4096}$; i.e., the space of bag-of-word representations. For each review, the non-zero vector components correspond to counts for words found within the review. Implicitly, this limits our vocabulary to the 4095 most frequent words and leaves one special token for out-of-vocabulary words. As noted, the label space $\mathcal{Y}$ is a binary space whose elements denote the sentiment of the review. We let $\mathcal{H}$ be either a linear model or a 4-layer FCN as before.
     \item \textbf{Discourse A (PDTB Labels)}: The Penn Discourse Treebank (PDTB) \citep{apx_prasad2008penn} is an NLP dataset containing a subset of Wall Street Journal articles from the Penn Treebank \citep{apx_marcus1993building} which are tagged with shallow discourse coherence relations (i.e., relations that hold only between the argument pairs and do not have any hierarchy or graph structure). These coherence relations can be explicitly signaled by discourse \emph{connectives} such as \emph{and}, \emph{so}, and \emph{but}, or could require the insertion of an \emph{implicit} connective. In this paper, we focus on the task of implicit discourse sense classification which is the most difficult task for discourse parsers. To form a DA dataset, we also used implicit relations from two parallel corpora: the TED-MDB \citep{apx_zeyrek2020ted} which contains tagged TED talks and the BioDRB \citep{apx_ramesh2010identifying} which contains tagged scientific articles. These three mentioned datasets form our component domains. Our feature space $\mathcal{X}$ is selected to be space of real-vectors of dimension 728. Within this space, we try three different feature representations for the discourse relations.\footnote{We only experiment \textit{within} each representation type and do not attempt to transfer \textit{across} different representations.} In the first two cases, the feature is made up of the argument pairs which have been concatenated and encoded using a BERT model \citep{apx_devlin2018bert}. We use either the pooled output or the average of the hidden states. In the last case, we use Sentence-BERT \citep{apx_reimers2019sentence} to encode our features. Our label space $\mathcal{Y}$ consists of the 4 level 1 discourse sense classes contained in the Penn Discourse Treebank. We let $\mathcal{H}$ be either a linear model or a 4-layer FCN as before.
     \item \textbf{Discourse B (GUM Labels)}: The GUM corpus \citep{apx_zeldes2017gum} contains text documents from 8 different genres: Academic, Biography, Fiction, Interview, News, Reddit, How-To, and Travel. These genres form our component domains. Documents within the corpus are annotated using the discourse framework of Rhetorical Structure Theory \citep{apx_mann1987rhetorical} in which discourse coherence relations are organized in a hierarchical tree structure. The sense hierarchy used for the GUM corpus is similar to that of the RST Discourse Treebank \citep{apx_carlson2003building}. In order to focus on coherence relations only between two argument pairs (without the additional hierarchical structure), we removed all relations where one or both nodes was not a leaf node. To form the label space, we mapped the twenty GUM labels to the conventional RST discourse treebank top-level labels where only three GUM labels did not have an existing mapping encoded in the RST. We mapped these three in the following manner, following \citet{apx_braud2017cross}: \emph{preparation} to BACKGROUND, \emph{justify} and \emph{motivation} to EXPLANATION, and \emph{solutionhood} to TOPIC-COMMENT. Given this mapping, our final label space $\mathcal{Y}$ consists of the 13 different RST discourse sense classes that were mapped to by the GUM corpus classes. Our features are encoded the same way as the PDTB features. We let $\mathcal{H}$ be either a linear model or a 4-layer FCN as before.
\end{enumerate}

\paragraph{Random Data Splits}
To simulate variability due to sampling and also to consider more mild forms of dataset shift, we split each component domain within each dataset into two disjoint sets of (roughly) equal size. So, for a dataset with 4 component domains, the new number of component domains will be 8. Some of these component domains \textit{should} now follow a fairly similiar distribution; i.e., splits coming from the same original component. This process is done randomly and all adaptation scenarios (see Section~\ref{sec:adaptation_pairs}) test 3 different seeds for this split.

\subsection{Model Training for Divergence Approximation, Adaptability Upperbounds, and Simple Algorithm (SA)}\label{sec:simple-training}
We train a number of deterministic models throughout our experimentation; e.g., for divergence approximation, to compute risks for the ranking task, and to compute upperbounds on $\lambda$. To avoid individual parameter selection for more than 12,000 models all trained in our experimentation, we use the optimization parameters given below in most cases. For Gibbs predictors, we use a slightly modified technique which is also discussed below. For the \textbf{DANN} algorithm, we use different parameters which were more carefully selected (details discussed in Section~\ref{sec:dann_details}). While this ``one size fits all'' approach is arguably simplistic, we found these settings worked well for a majority of cases in our preliminary experiments. For divergence approximation and upperbounds on adaptability this is visible in the main text results. In Appendix~\ref{sec:sanity_check}, we also report statistics on the transfer (target) error of some hypotheses trained to minimize error on the source sample (i.e., the Simple Algorithm \textbf{SA}). These provide a sanity check that our simple optimization procedure is indeed selecting non-trivial hypotheses in a majority of cases. Notably, the point of this work is to study global trends rather than to achieve optimal performance on any one dataset. The ``one size fits all'' approach we take is reflective of this; it allows us to use our limited computational resources to study \textit{more} datasets and models, rather than do rigorous parameter search on just a few.

\paragraph{Optimization Parameters} All models are trained using SGD on an NLL loss with momentum set to $0.9$. The NLL loss is sometimes weighted to correctly replicate the importance of multiple risks (e.g., when minimizing a sum of two risks). For example, if we have the objective $\min_h \mathbf{R}_S(h)+ \mathbf{R}_T(h)$ and $S$ has more samples than $T$, the NLL loss will weight examples in $T$ higher to give them equal importance during optimization, as described by the objective.  We start training with a learning of $1 \times 10^{-2}$ for 100 epochs. Then, we train for another 50 epochs using a learning rate of $1 \times 10^{-3}$. If a model ever achieves a training error lower than $5 \times 10^{-4}$, we terminate training. In all cases, we use a batch size of $250$.

\paragraph{Gibbs Predictors} To learn a Gibbs predictor (stochastic model) $\mathbb{Q}$ we need to use a slightly different approach. In all cases, $\mathbb{Q}$ will be a multivariate normal distribution with diagonal covariance and we will minimize Gibbs risk on a source sample $S$ with intention to transfer to a target sample $T$. We use PAC-Bayes-by-Backprop (PBB) to learn the parameters of our normal distribution. PBB is an SGD-based technique proposed by \citet{apx_perez2021tighter} to learn stochastic models that optimize PAC-Bayes bounds. The approach requires specification of a particular PAC-Bayes bound to use as the objective and a particular distribution $\mathbb{P}$ to use as the prior. For the former, we use the variational bound proposed by \citet{apx_dziugaite2021role}. For the latter, we use a multivariate normal distribution: the mean is a (trained) deterministic model (i.e., its parameter vector) and the covariance matrix is $\sigma \mathbf{I}$ where $\sigma = 0.01$ and $\mathbf{I}$ is the identity matrix. We train the deterministic model to minimize the error on $S$ using the same optimization parameters discussed previously. Note, this may seem taboo to one familiar with PAC-Bayes, since the prior $\mathbb{P}$ is typically required to be independent of the data used in the bound. Contrary to this, in our setting, it is \textit{perfectly valid} to select $\mathbb{P}$ based on the data in (only) $S$. This is clear in the proofs of Theorem~\ref{thm:pb-bound} and Theorem~\ref{thm:pb-bound-efficient} because the prior $\mathbb{P}$ is only used to bound the generalization gap between $T$ and $\mathbb{T}$. Thus, this choice is reflective of a realistic scenario where one wishes to compute a PAC-Bayes bound with $\mathbb{Q}$. The approach we describe essentially corresponds to the idea of using a data-dependent prior (see Section~\ref{sec:ext_related}). The prior $\mathbb{P}$ is learned using data that is not used in any part of the bound which depends on $\mathbb{P}$. For optimization parameters of PBB not discussed here (e.g., learning rate), we default to the previously discussed choices.

\subsection{DANN Model Training}
\label{sec:dann_details}
For \textbf{Digits}, we study a PAC-Bayes variant of the invariant feature learning algorithm \textbf{DANN} (Domain Adversarial Neural Network) proposed by \citet{apx_ganin2015unsupervised}. The output of our variation is a Gibbs predictor $\mathbb{Q}$, so we again employ PBB as discussed above in Section~\ref{sec:simple-training}. While many parameters are similar to those used above -- including the prior and the PAC-Bayes bound --, we highlight some differences here. Most notably, we re-weight the KL divergence in the PAC-Bayes bound by a dampening factor to reduce its regularizing impact during training. For example, if the KL divergence was 46K and the dampening factor was 0.1, the effective KL divergence during training is 4.6K. We explore a range of different dampening factors to get a breadth of different ``complexities'' for interpretation; i.e., this variability produces the movement along the horizontal axes of Figure~\ref{fig:dann-all}. In our experiments, we let the dampening factor range in the set $\{0.1, 0.05, 0.01, 0\}$. We did not use any dampening factor to recover the original PAC-Bayes bound (i.e., 1) because we found this setting to be too restrictive in preliminary experiments. Due to the increased training time involved in this parameter sweep, we down-sampled the \textbf{Digits} dataset discussed above so that neither $S$ nor $T$ have more than 5K examples. We selected the learning rate by manual inspection, varying the learning rate (and number of epochs accordingly) until we did not observe frequent gradient explosion / vanishing. We ended up using an initial learning rate of $1 \times 10^{-3}$ for 112 epochs (75\% of 150) and $5 \times 10^{-4}$ for the remaining 38 epochs. We also reduced the batch size to 128, but most other parameters remained the same. It is important to note that the instability we experienced (i.e., related to gradient explosion / vanishing) is somewhat common when training adversarial methods such as DANN. As an additional measure to combat this issue, we slowly eased in the adversarial loss by weighting it using the parameter $\beta_p = 2 / (1 + \exp(-10p))$ where $p$ is the progress ratio of the current epoch in training; e.g., $p = 0.1$ corresponds epoch 15 out of 150. This approach was first proposed by \citet{apx_ganin2015unsupervised} for the same purpose. Besides the adversarial loss, we also multiply the KL divergence dampening factor by this weight to ease in the regularization component as well. As a final measure, we used multiple restarts with a new neural network initialization (up to 25 attempts), which proved to be the most effective measure. Among about 32K statistics computed during these experiments, only 21 statistics were unable to be computed due to instability. These consisted of restricted model-independent divergences (i.e., using class $[\mathcal{H}\Delta\mathcal{H}]_\mu$) and were ignored in plots. Roughly 5\% of the data points still had ``extreme'' values that did not match any other trend, so we removed these in Figure~\ref{fig:dann-all} to help with visual interpretation. 
\subsection{Adaptation Pairs}
\label{sec:adaptation_pairs}
We now discuss the different adaptation scenarios we consider. Instances of each scenario produce the collection of $(S,T)$ pairs we consider in our histograms in the main text. Recall, we randomly split each component domain into two halves (see Section~\ref{sec:datasets_and_models}). This will be important for understanding our adaptation scenarios. 
\paragraph{Single-Source} For all datasets except \textbf{Discourse B}, we consider a single-source adaptation scenario: each component domain in a dataset is paired with each distinct component domain. So, for a dataset with 8 components, this forms 64 $(S,T)$ pairs. For example, one pair might take $S$ to be the first random half of SVHN and $T$ to be the first random half of USPS. Another pair might take $S$ to be the first random half of SVHN and $T$ to be the \textit{second} random half of SVHN. So, as we see from this example, this implies that components derived from the same domain (i.e., through our random splitting procedure) will be paired. Note, these should follow a fairly similar (or identical) distribution. This is purposeful and provides a number of instances of \textbf{within-distribution} shift. These milder forms of shift allow us to test a broader range of realistic scenarios. The random splitting procedure also allows us to test variability in the outcome of a transfer task due to sampling; e.g., the first and second random half of SVHN will both be paired with every other dataset. 
\paragraph{Multi-Source} We also consider multi-source scenarios for all datasets except \textbf{Discourse A}. In these cases, we group all but one component of the dataset into a single pooled sample. The single component which was left out is chosen to be the target. So, for a dataset with 4 components, this forms 4 $(S,T)$ pairs. For example, for PACS, one pair might take $S$ to be the union of Art, Cartoon, and Sketch while $T$ consists of only Photo. \textbf{Importantly}, we only use \textit{one} of the random splits from each component type; e.g., for PACS, although we split photo into two disjoint sets, we only use one of these two sets. Otherwise, in every $(S,T)$ pair, $S$ would contain some data coming from a similar distribution as $T$. Informally speaking, this likely to weaken the adaptation difficulty. Note, the adaptation bounds we give implicitly cover multi-source contexts, since we can view the single source $\mathbb{S}$ as a mixture distribution.
\paragraph{Digits-Specific Scenarios} The \textbf{Digits} dataset is particular interesting because the feature space $\mathcal{X}$ is well-understood by humans. Thus, we can use our experience to design some natural distribution shifts. In particular, we consider the case where $S$ is some component of \textbf{Digits} and $T$ is the same sample except every image is randomly rotated up to $360^\circ$. We also consider the case where $S$ is some component of \textbf{Digits} and $T$ is the same sample except every image is blurred with random white noise. We can also consider a very unnatural shift. In particular, we consider transfer to randomly generated data. Here, $S$ is some component of \textbf{Digits} (as before) and every image in $T$ is a $28 \times 28$ grid of randomly generated pixels which is assigned a random label. For these scenarios, we use the entirety of the components in the \textbf{Digits} sets without doing any random splitting.

\subsection{Details for For Figure~\ref{fig:lambda}}\label{sec:exp_details_ada}
Each dataum in Figure~\ref{fig:lambda} corresponds to an upperbound for one of the adaptation pairs described in Section~\ref{sec:adaptation_pairs} using one of the compatible hypothesis spaces described in Section~\ref{sec:datasets_and_models}. We describe the process for computing each type of upperbound below. In all cases, we compute the upperbound using 3 different random seeds; i.e., this will effect things like the model-training and subsequently the final bound. We report the smallest upperbound of these seeds. This is logical since the smallest upperbound is still a valid upperbound. In case of $\lambda$, this is actually overly optimistic since the confidence parameter should be changed to account for all 3 bounds.  

\paragraph{Upperbound for $\tilde{\lambda}$} This bound is computed just as described in the main text. For each adaptation pair $(S,T)$ and each hypothesis space $\mathcal{H}$ which is compatible with $S$ and $T$, we train a model to minimize the summed risks on $S$ and $T$ using the approach described in Section~\ref{sec:simple-training}. If this approach returns the hypothesis $h$, we report $\mathbf{R}_S(h) + \mathbf{R}_T(h)$. As noted in the main text, this is a valid upperbound for $\tilde{\lambda}$.

\paragraph{Upperbound for $\lambda$}
 For each adaptation pair $(S,T)$ and each hypothesis space $\mathcal{H}$ which is compatible with $S$ and $T$, we randomly split $S$ and $T$ using an 80/20 train/test split. Denote train splits for $S$ and $T$ by $S_\mathrm{tr}$ and $T_\mathrm{tr}$, respectively. Denote the test splits for $S$ and $T$ by $S_\mathrm{ho}$ and $T_\mathrm{ho}$, respectively. We train a model to minimize summed risks on $S_\mathrm{tr}$ and  $T_\mathrm{tr}$ using the approach described in Section~\ref{sec:simple-training}. If this approach returns the hypothesis $h$, we then report the quantity
 \begin{equation}
     \mathbf{R}_{S_\mathrm{ho}}(h) + \mathbf{R}_{T_\mathrm{ho}}(h) + \sqrt{\ln (4 / \delta) / (2m)} + \sqrt{\ln (4 / \delta) / (2n)} 
 \end{equation}
where $m = |T_\mathrm{ho}|$, $n = |S_\mathrm{ho}|$, and $\delta = 0.05$. This is a valid upperbound for $\lambda$ which holds (i.e., prior to observing data) with probability $1 - \delta$. It is easily derived using Hoeffding's Inequality to bound both $\Delta_h(S_\mathrm{ho}, \mathbb{S})$ and $\Delta_h(T_\mathrm{ho}, \mathbb{T})$, and then, using Boole's Inequality to combine the bounds.
\subsection{Details for For Table~\ref{tab:divapprox}}\label{sec:exp_details_divapprox}
Each dataum used to compute the correlations in Table~\ref{tab:divapprox} corresponds to a divergence approximation for an $(S,T,h)$ triple. The datasets $S$ and $T$ are given by one of the adaptation pairs described in Section~\ref{sec:adaptation_pairs}, and conceptually, we pick $h$ to be the hypothesis whose error we would like to bound in the simulated transfer from $S$ to $T$. Specifically, for each adaptation pair $(S,T)$ and each compatible hypothesis space $\mathcal{H}$, we select $h$ using \textbf{SA}. We then approximate either the $\mathcal{H}\Delta\mathcal{H}$- or $h\Delta\mathcal{H}$-divergence using the (appropriate) technique described in Section~\ref{sec:method_approx}. The Spearman rank correlation we report compares the $\mathcal{H}\Delta\mathcal{H}$- and the $h\Delta\mathcal{H}$-divergence to the error-gap $\Delta_h(S,T)$ defined in Eq.~\eqref{eqn:error_gap} over all adaptation pairs $(S,T)$ in Section~\ref{sec:adaptation_pairs}, all compatible hypothesis spaces $\mathcal{H}$ discussed in Section~\ref{sec:datasets_and_models}, and all 3 seeds.
\subsection{Details for For Estimation of Flatness}\label{sec:exp_details_flatness}
Each datum discussed in the main text paragraph \textbf{Do Flat Regions Transfer?} corresponds to an estimate for an $(S,T,\mathbb{Q})$ triple. The datasets $S$ and $T$ are given by one of the adaptation pairs described above, and as before, we pick $\mathbb{Q}$ to be the Gibbs predictor whose error we would like to bound in the simulated transfer from $S$ to $T$ (i.e., using \textbf{SA}). We then estimate as described in the main text. For each $(S,T)$ pair and compatible hypothesis space $\mathcal{H}$, we repeat this procedure with 3 seeds to control for variability in the selection of $\mathbb{Q}$. We report all seeds in the histogram in Figure~\ref{fig:rho-hist}. Also, the mean and standard deviation reported in the main text are computed using all data in the histogram.
\subsection{Details for Figure~\ref{fig:dann-all}}
Each datum in Figure~\ref{fig:dann-all} corresponds to statistics computed for an $(S, T, \mathbb{Q})$ triple. As noted, $S$ and $T$ are restricted to be samples from the \textbf{Digits} dataset. Further, we only consider out-of-distribution adaptation scenarios as indicated in each subfigure title. The Gibbs predictor $\mathbb{Q}$ is selected using the training details discussed in Section~\ref{sec:dann_details} and the statistics are reported as described in Thm.~\ref{thm:pb-bound-efficient} and Cor.~\ref{cor:pb-bound-efficient}.
\section{Additional Experimental Results}
\label{sec:ext_results}
\begin{figure}
    \centering
    \includegraphics[width=\textwidth]{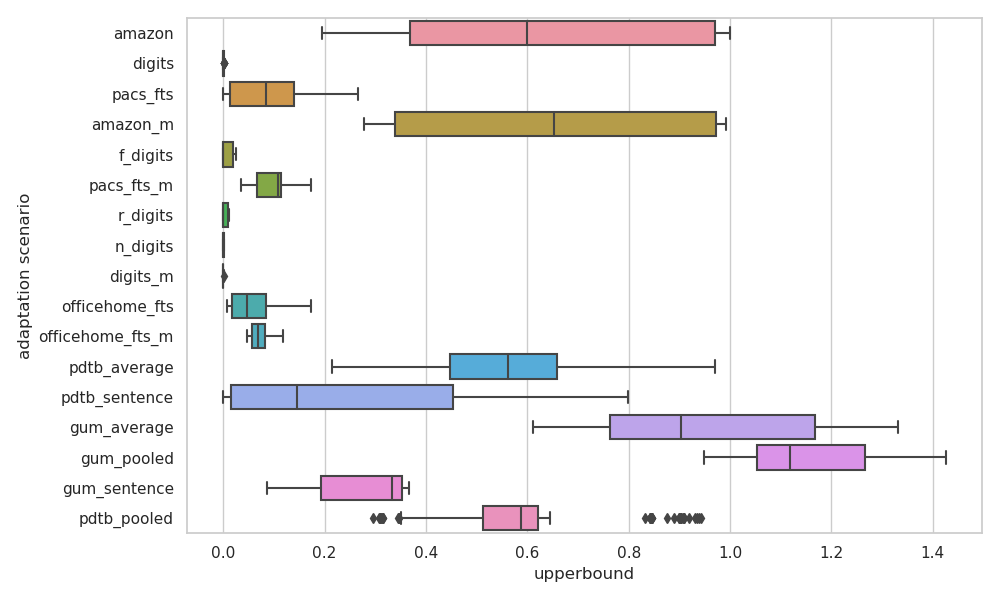}
    \caption{Boxplots for upperbounds on $\tilde{\lambda}$ for individual datasets. The appendage \texttt{\_m} denotes the \textbf{multi-source} setup, otherwise it is \textbf{single-source}. For digits, we prepend \texttt{r\_}, \texttt{n\_}, or \texttt{f\_} to denote transfer to rotated, noisy, or randomly generated (fake) data as discussed in Section~\ref{sec:adaptation_pairs}. As may be inferred, \texttt{pdtb} corresponds to \textbf{Discourse A} and \texttt{gum} corresponds to \textbf{Discourse B}. Appendages for these indicate the type of BERT features used.}
    \label{fig:grouped_ubs}
\end{figure}
\begin{figure}
    \centering
    \includegraphics[width=\textwidth]{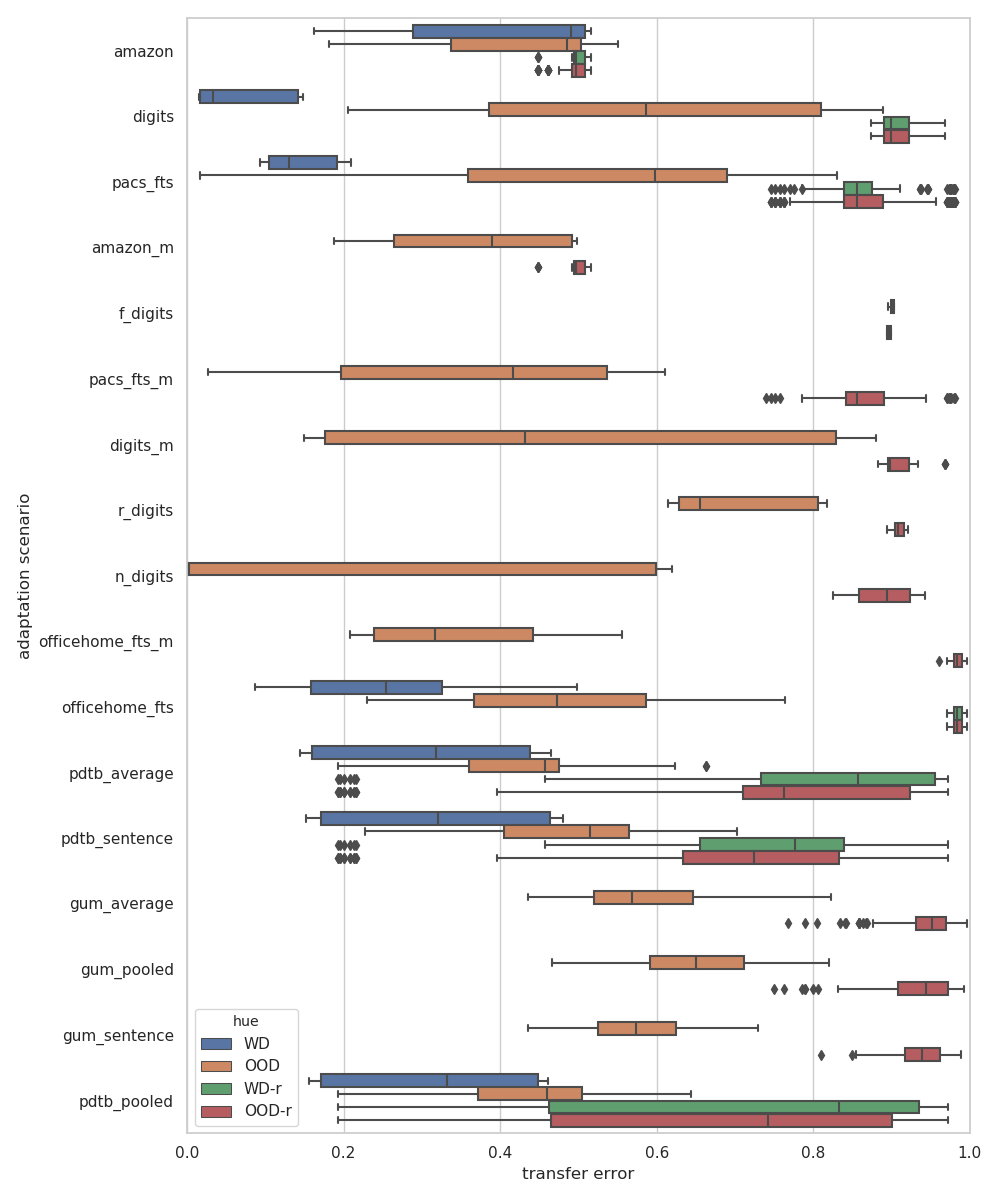}
    \caption{Dataset names are as in Figure~\ref{fig:grouped_ubs}. Hues correspond to within-distribution \textbf{WD} and out-of-distribution \textbf{OOD} based on whether the target $T$ is drawn from the same component domain as the source $S$. Notice, \textbf{WD} is not available for multi-source setups and certain Digits setups. This is simply a property of these adaptation scenarios. In these cases, we still report the \textbf{OOD} error. The appendage \texttt{-r} denotes the hypothesis is randomly initialized and not trained. These experiments provide a point of reference for comparison.}
    \label{fig:grouped_transfer_error}
\end{figure}

\subsection{Upperbounds on Adaptability (Dataset Specific)}\label{sec:adaptability_boxp}
We also found it interesting to consider our sample-dependent adaptability in a more problem-specific context. This reveals to us that \textbf{PACS} and \textbf{Office-Home} have the larger upperbounds of computer vision datasets. It also reveals that most upperbounds above about $0.3$ are due to NLP datasets. Informally speaking, this is sensible as the NLP tasks we consider have higher uncertainty in the labeling functions. Results are shown in Figure~\ref{fig:grouped_ubs}. We use the same aggregate data as Figure~\ref{fig:lambda}.

\subsection{Transfer Error of Trained and Random Hypotheses}\label{sec:sanity_check}
In Figure~\ref{fig:grouped_transfer_error}, we show transfer error (i.e., the error on $T$) of hypotheses trained using \textbf{SA}. These are precisely the hypotheses used to compute the error-gap $\Delta_h(S,T)$ when reporting Spearman rank correlation. As noted, we use a standardized optimization procedure to forego parameter selection on the more than 12,000 models we train. Thus, we are not interested in optimal performance in any case. Instead, Figure~\ref{fig:grouped_transfer_error} primarily serves as a sanity check to make sure the hypotheses we use are somewhat reflective of those which might be used in a practical scenario. That is, we would like to confirm that these hypotheses have learned something non-trivial (at least on the source domain). To illustrate this, for all datasets, we also report the error of randomly initialized hypotheses which have not been trained. This provides a point of reference. It is easy to see our trained hypotheses are typically far more effective than the untrained random initializations.

For the discourse datasets with PDTB labels, we observe the error of the random initializations is somewhat harder to interpret. For this reason we compare to a related work. In particular, \citet{apx_kishimoto-etal-2020-adapting} achieve an error rate of $\approx 0.38$ on a comparable (within-distribution) discourse sense classification task. We observe our within-distribution PDTB results (Discourse A) are frequently better than this.

\subsection{Comparison to Germain et al.}
\label{sec:comp2germain}
\begin{figure}
    \centering
    \includegraphics[width=.7\columnwidth]{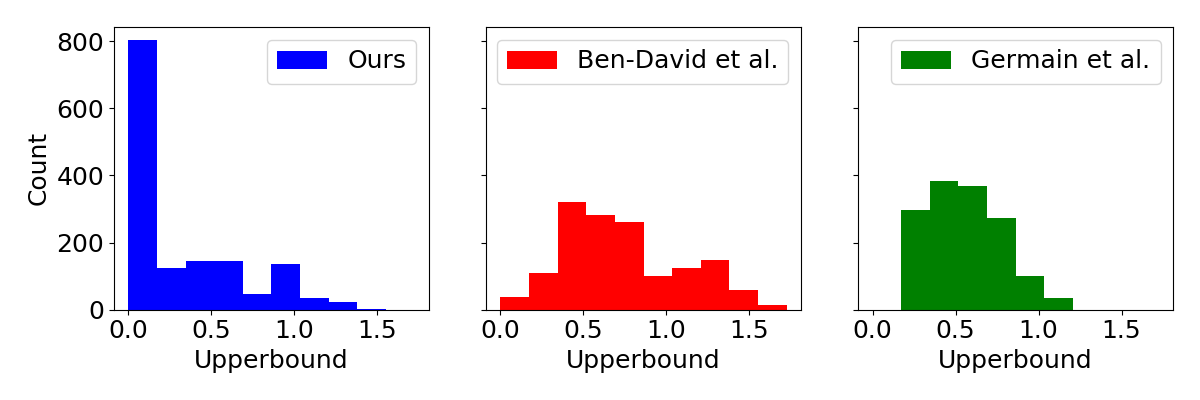}
    \caption{\small Upperbounds for sample-dependent (left), sample-independent (center), and binary PAC-Bayes variant (right) of adaptability. Each datum describes unique $(S, T, \mathcal{H})$.}
    \label{fig:g_lambda}
\end{figure}
While we know, theoretically, the PAC-Bayes bound of \citet{apx_germain2020pac} is not valid for the multiclass setting, we also study this question empirically. In Figure~\ref{fig:g_lambda}, we present a sample-dependent variation of the adaptability proposed by Germain et al.\footnote{Some proof-techniques we employ in Thm.~\ref{thm:pb-bound} can be applied to derive a sample-dependent variation of Thm.~\ref{thm:germain2020pac}, instead.} because of our previous (positive) results on sample-dependence. Even with this upgrade, the adaptability of Germain et al. is not able to capture the same useful information as our multiclass sample-dependent adaptability. Further, conducting a similar experiment as in Table~\ref{tab:divapprox}, we find the divergence term of Germain et al. has low rank correlation (\textbf{0.27} on all data). These empirical results confirm the hypothesis of \citet{apx_germain2020pac} that their PAC-Bayesian theory of adaptation in binary settings is not easily extend to the multiclass setting. This is especially true in comparison to the positive outcomes observed under our theory. The experimental details for these results are provided below.
\paragraph{Upperbound for Adaptability of \citet{apx_germain2020pac}}
The historgram in Figure~\ref{fig:g_lambda} shows a histogram of upperbounds on a sample-dependent variation of the adaptability of \citet{apx_germain2020pac} given in Thm.~\ref{thm:germain2020pac}. We compute the upperbounds using the same setup as described in Appendix~\ref{sec:exp_details_ada}. Because we do not have full access to $\mathbb{Q}$, we instead estimate this term with a finite sample $Q = (H_i)_i \sim \mathbb{Q}^k$.
Using Linearity of $\mathbf{E}$ and Hoeffding's Inequality, we have the following bound on Germain et al.'s adaptability (our sample-dependent variant) with i.i.d. sample $(H_{i,1}, H_{i,2})_i \sim (\mathbb{Q} \times \mathbb{Q})^k$
\begin{equation}\small\label{eqn:lambda_rho_ub}
\begin{split}
     \Bigg \lvert \ k^{-1}\sum\nolimits_i \underset{{X,Y}}{\mathbf{E}}[1[H_{i,1}(X) \neq Y] \cdot 1[H_{i,2}(X) \neq Y]] - k^{-1}\sum\nolimits_i \underset{{\tilde{X},\tilde{Y}}}{\mathbf{E}}[1[H_{i,1}(\tilde{X}) \neq \tilde{Y}] \cdot 1[H_{i,2}(\tilde{X}) \neq \tilde{Y}]] \ \Bigg \rvert
     + \sqrt{\frac{2 \ln (2 / \delta)}{k}}.
 \end{split}
\end{equation}
Here, we pick $\delta = 0.05$ as before and use $k = 100$. This gives a valid bound for which holds with probability $1 - \delta$ (i.e., prior to seeing the samples from $\mathbb{Q}$). We select the Gibbs predictor $\mathbb{Q}$ using \textbf{SA} as before.
\paragraph{Divergence of \citet{apx_germain2020pac}} We also approximate the divergence of Germain et al. to compare to the $\mathcal{H}\Delta\mathcal{H}$- and $h\Delta\mathcal{H}$-divergence in terms of model selection. As noted, the comparison is made through Spearman rank correlation with error-gap $\Delta$ using the same experimental setup as in Appendix~\ref{sec:exp_details_divapprox}. Since we are not aware of an analytic solution for this divergence (in case of neural networks), we approximate the divergence term of Germain et al. using a random sample $Q \sim \mathbb{Q}^k$ with $k = 100$. Here, $\mathbb{Q}$ is a distribution over $\mathcal{H}$ selected, again, using \textbf{SA}. We do this for each adaptation pair $(S,T)$ and each compatible hypothesis space $\mathcal{H}$. The final reported correlation compares the approximated divergence to $\Delta_Q(S,T)$ over all adaptation pairs, all compatible hypothesis spaces, and all 3 seeds.
\subsection{Additional DANN Results}\label{sec:ext_dann_res}
\begin{figure}
    \centering
    \includegraphics[width=.7\columnwidth]{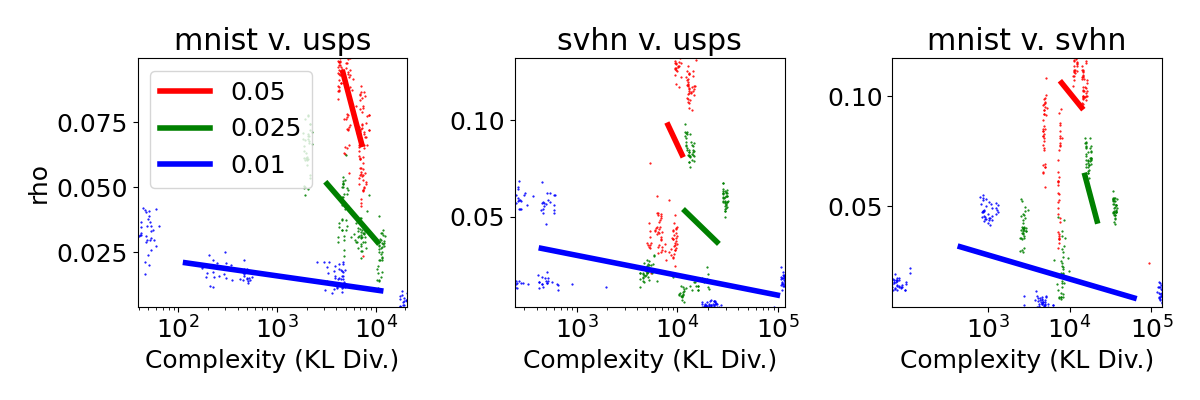}
    \caption{Estimates for $\rho$ while using \textbf{DANN} for various choices of the prior variance parameter $\sigma$. Solid line shows median, while scatter shows 95\% or more of data. Each datum describes unique $(S, T, \mathbb{Q})$. As a function of complexity, we expect $\rho$ to be smaller for more complex solutions. For example, in the formula for $\mathrm{KL}$-divergence between Gaussian distributions, similarly concentrated distributions will have high $\mathrm{KL}$-divergence as their variances decrease (i.e., holding all else constant). Sensibly, smaller variance gives more concentrated $\mathbb{Q}$, which helps to ensure small $\rho$ as well. This relationship (between $\rho$ and complexity) is observed in the above and is similar to our findings in the main text on adaptability after \textbf{DANN}.}
    \label{fig:rho-dann}
\end{figure}
\begin{figure}
    \centering
    \includegraphics[width=.7\columnwidth]{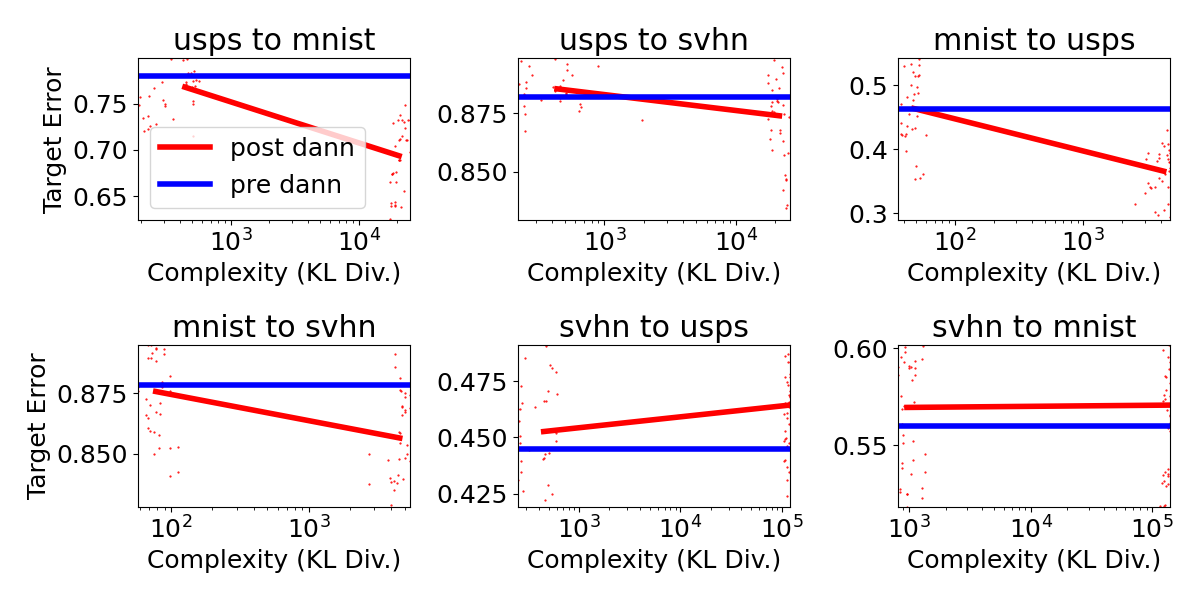}
    \caption{Estimates of target error $\mathbf{R}_T(\mathbb{Q})$ (after \textbf{DANN}) compared to prior error $\mathbf{R}_T(\mathbb{P})$ (before \textbf{DANN}, i.e. using \textbf{SA}). Solid line shows median, while scatter shows 95\% or more of data. Each datum describes unique $(S, T, \mathbb{Q})$. More complex solutions achieve lower error as expected. \textbf{DANN} is effective at reducing target error in many cases.}
    \label{fig:err}
\end{figure}
In Figure~\ref{fig:rho-dann}, we show the effect of \textbf{DANN} on $\rho$. This confirms our takeaway in the main text that, as a function of sample complexity, $\rho$ behaves like adaptability. Also, we see that increasing the prior variance makes flatness less likely, since this indirectly controls the variance of $\mathbb{Q}$, which is regularized to be similar to the prior via PBB (see Section~\ref{sec:dann_details}). Intuitively, it is easier to find small flat-minima than very large flat-minima. In Figure~\ref{fig:err}, we show the target (transfer) error of solutions trained using \textbf{DANN} compared to solutions trained using \textbf{SA}. The error rates may appear unusually high to familiar readers (e.g., compared to \citet{apx_ganin2015unsupervised}), but this is likely a result of the down-sampling we do to save training time (see Section~\ref{sec:dann_details}). Since \textbf{DANN} is a more sophisticated adaptation algorithm, we expect it to learn more about the target than \textbf{SA}, and indeed, it does in many contexts. Thus, this result also serves to validate our empirical setup for applying \textbf{DANN}. Lastly, as before, we can interpret the target error as a function of sample complexity: unconstrained solutions are able to achieve lower error than constrained solutions. Due to the higher complexity, these solutions may not generalize well.

\subsection{Detailed Visualization of Data in Table~\ref{tab:divapprox}}
In some instances, we observe poor correlation of the proposed divergence terms with the error-gap. For example, in Table~\ref{tab:divapprox}, poor correlation is observed on the \textbf{Digits} dataset. Poor correlation is also observed on the \textbf{PACS+OH} dataset for the model-dependent divergence. To study and understand these errors in detail, we visualize heatmaps (i.e., 2d histograms) in Figure~\ref{fig:heatmaps}. Histogram counts illustrate counts of the individual data pairs used to compute correlation in Table~\ref{tab:divapprox}; i.e., between a particular approximation of divergence and the corresponding error-gap. Please, see Section~\ref{sec:exp_details_divapprox} for details on the data used for Table~\ref{tab:divapprox}.

Results show the poor performance on the \textbf{Digits} dataset is likely due to insensitivity of the divergence approximation to changes in data sample and hypothesis. In particular, there is significant concentration of the divergence approximations near 1. In the case of the model-independent divergence, we also observe some artificially low approximations (i.e., near 0) compared to the error-gap; this illustrates poor approximation. On the \textbf{PACS+OH} dataset, the poor performance of the model-dependent divergence is best explained by comparing to the data-points for the model-independent divergence. While the model-dependent divergence is more variable and sensitive to data/model changes as one would expect, we see a high density of anti-correlated measurements on the \textbf{PACS+OH} data. Specifically, there is a cluster of cases where the divergence is near 1 with absolute change in error about 0.4 and another cluster of cases where the divergence is only about 0.8 with absolute change in error higher at 0.6. Aptly, the divergence does not perform well at ranking in this case.

These more nuanced results speak to the conservative nature of bounds (and their contained terms), in general. In particular, upperbounds are subject to “false positives” -- in which the actual bound is high, but the quantity controlled is low; e.g.,1 is a valid bound on 1 and so are 2, 10, 50, and 1000. While this undesirable property impacts the \textbf{Digits} and \textbf{PACS+OH} cases, it is also worth mentioning that the divergences perform well in many other cases. Depending on the application, a conservative measure of performance change may actually be desirable.

\begin{figure}
    \centering
    \includegraphics[width=\columnwidth]{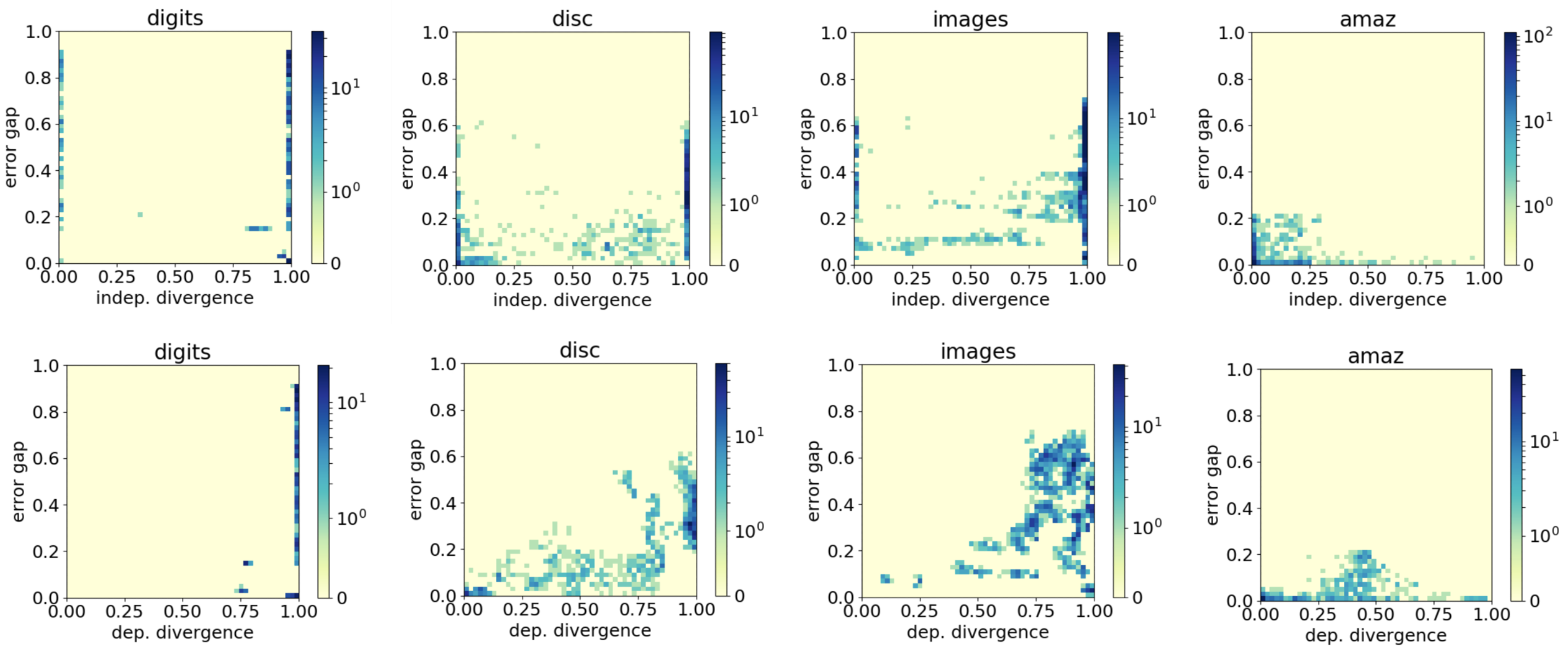}
    \caption{Heatmap (i.e., 2d histogram) showing counts for data used to compute correlations in Table~\ref{tab:divapprox}. \textbf{images} corresponds to the \textbf{PACS+OH} dataset.}
    \label{fig:heatmaps}
\end{figure}
\end{document}